\documentclass{article} 
\usepackage[preprint]{colm2026_conference}

\usepackage{microtype}
\usepackage{hyperref}
\usepackage{url}
\usepackage{booktabs}

\usepackage{float} 
\usepackage{placeins}
\usepackage{natbib}

\usepackage{amsmath,amsfonts,bm}

\def\eqref#1{equation~\ref{#1}}

\def\1{\bm{1}}

\DeclareMathAlphabet{\mathsfit}{\encodingdefault}{\sfdefault}{m}{sl}
\SetMathAlphabet{\mathsfit}{bold}{\encodingdefault}{\sfdefault}{bx}{n}

\usepackage{graphicx}
\usepackage{subcaption}
\usepackage{natbib}

\usepackage{xcolor}
\usepackage{colortbl}
\usepackage{hyperref}
\definecolor{DarkBlue}{rgb}{0.1,0.1,0.5}
\definecolor{DarkGreen}{rgb}{0.1,0.5,0.1}
\definecolor{deepyellow}{RGB}{218, 174, 42}
\hypersetup{
	colorlinks   = true,
	linkcolor    = DarkBlue, 
	urlcolor     = DarkBlue, 
	citecolor    = DarkGreen 
}

\usepackage{xfrac}
\usepackage{amsthm}
\usepackage{thmtools,thm-restate}
\usepackage{amsmath,amsfonts,dsfont}
\usepackage[mathscr]{euscript}
\usepackage{booktabs} 

\usepackage{multirow}
\usepackage{svg}
\usepackage{makecell}
\usepackage{pgfplots}
\pgfplotsset{compat=1.18} 
\usepgfplotslibrary{groupplots}
\usetikzlibrary{arrows.meta, positioning, shapes.geometric, decorations.pathreplacing, backgrounds}

\definecolor{DarkBlue}{rgb}{0.3,0.3,0.70}
\definecolor{azure}{rgb}{0.0, 0.5, 1.0}
\definecolor{darkcerulean}{rgb}{0.03, 0.27, 0.49}
\definecolor{denim}{rgb}{0.08, 0.38, 0.74}
\definecolor{DarkGreen}{rgb}{0.3,0.7,0.3}
\definecolor{lighter-gray}{gray}{0.95}
\hypersetup{
	colorlinks   = true,
	linkcolor    = DarkBlue, 
	urlcolor     = darkcerulean, 
	citecolor    = denim 
}
\definecolor{AlgHighlight}{HTML}{228B22}

\newtheoremstyle{thmstyle}
{0.5em} 
{0.15em} 
{} 
{} 
{\bfseries} 
{.} 
{.5em} 
{} 

\theoremstyle{thmstyle} 
\newtheorem{theorem}{Theorem}

\theoremstyle{definition}

\usepackage[most]{tcolorbox}     
\definecolor{turquoise}{HTML}{30D5C8}  

\newtcolorbox{keytakeawaysbox}{%
  enhanced,
  colback=white,
  colframe=turquoise,
  coltitle=black,                    
  boxrule=0.8pt,
  arc=2mm,                           
  left=2mm,right=2mm,top=0.5mm,bottom=0.5mm,
  title={\small\bfseries Key Takeaways},
  attach boxed title to top left={
    yshift=-\tcboxedtitleheight/2,
    xshift=6pt
  },
  boxed title style={
    colback=white,
    colframe=turquoise,
    boxrule=0.8pt,
    arc=2mm                          
  },
}

\definecolor{takeawayframe}{RGB}{45,82,160}   
\definecolor{takeawayback} {RGB}{245,248,252} 

\newtcolorbox{takeawaybox}[1][]{%
  enhanced, breakable,
  sharp corners=south,         
  colback   = gray!5!white,
  colframe  = gray!10!white,
  boxrule   = 0.7pt,
  left      = 1mm,             
  right     = 1mm,
  top       = 0.5mm,
  bottom    = 0.5mm,
  before skip=0pt,             
  after skip =0pt,             
  before upper=\itshape,       
  #1}                          

\theoremstyle{remark}

\renewcommand{\1}{ \mathds{1}}

\newcommand{\x}{\mathit{x}}

\renewcommand{\emptyset}{\varnothing}

\newcommand{\CommentLines}[1]{}

\usepackage{array}
\newcolumntype{x}[1]{>{\centering\let\newline\\\arraybackslash\hspace{0pt}}m{#1}}
\usepackage{colortbl}

\usepackage{wrapfig}
\usepackage{svg}

\usepackage{amsmath} 
\usepackage{amssymb} 
\usepackage{colortbl} 
\usepackage{enumitem}

\definecolor{green}{HTML}{C6EFCE}
\definecolor{red}{HTML}{FFC7CE}
\definecolor{yellow}{HTML}{FFEB9C}

\usepackage{algorithm}
\usepackage{algorithmic}

\usepackage{listings}

\definecolor{darkgreen}{rgb}{0, 0.6, 0}

\definecolor{titlecolor}{RGB}{70,130,180}  
\definecolor{titlecolor}{RGB}{45,82,160}    
\definecolor{inputcolor}{RGB}{0,100,0}     
\definecolor{inputcolor}{RGB}{25,70,25}
\definecolor{outputcolor}{RGB}{139,69,19}  
\definecolor{stepcolor}{RGB}{25,25,112}    
\definecolor{mathcolor}{RGB}{128,0,128}    
\definecolor{inputcolor}{RGB}{25,111,61}     
\definecolor{outputcolor}{RGB}{155,89,182}   
\definecolor{outputcolor}{RGB}{192,57,43}    
\definecolor{outputcolor}{RGB}{170,85,0}     
\definecolor{outputcolor}{RGB}{45,82,160}    
\definecolor{outputcolor}{RGB}{75,61,136}    

\newcommand{\ttsd}{\textsc{TTEL}}

\makeatletter
\def\blfootnote{\gdef\@thefnmark{}\@footnotetext}
\makeatother

\definecolor{algblue}{RGB}{31,78,121}
\definecolor{alggreen}{RGB}{56,118,29}
\definecolor{algorange}{RGB}{180,95,6}
\definecolor{algpurple}{RGB}{102,51,153}

\newcommand{\dg}[1]{}
\newcommand{\rsc}[1]{}
\usepackage{cleveref}

\title{Test-Time Scaling via Error Localization}

\author{\textbf{Rajiv Shailesh Chitale, Rahul Madhavan, Taneesh Gupta,} \\
\textbf{Deepanway Ghosal, Aravindan Raghuveer} \\[0.15cm]
Google DeepMind
\vspace{-1.4em}
}

\begin{document}

\ifcolmsubmission
\linenumbers
\fi

\maketitle
\blfootnote{\hspace{-0.7cm} \texttt{\{rchitale,rahulmadhavan,tanishg,deepanway,araghuveer\}@google.com}}

\vspace{-0.1in}
\begin{abstract}
Scaling inference-time computation has emerged as a reliable method to improve the performance of large language models on complex reasoning and programming tasks. However, standard approaches such as independent sampling and sequential multi-turn refinement operate without token-level credit assignment, resulting in computational inefficiency, since valid reasoning prefixes are frequently discarded. In this work, we introduce Test-Time Scaling via Error Localization ({\ttsd}), an inference-time algorithm that utilizes fixed or environment feedback to perform token-level error localization. By comparing conditional probabilities under informed feedback against a null-context baseline, {\ttsd} isolates the step at which an error occurred. The algorithm then truncates the trajectory and branches a new generation, maximally reusing the valid prefix. Extensive evaluations demonstrate that {\ttsd} establishes strictly dominating Pareto frontiers across sequential reasoning domains, measured by pass-at-k vs. generated-token cost. With Qwen3-8B on LiveCodeBench, {\ttsd} attains a pass@64 of 71.0\% while generating approximately half as many tokens as independent sampling (360.4k vs.\ 735.0k). Generalizing to math benchmarks AIME-2025 and HMMT-2025, {\ttsd} cleanly outperforms competing test-time baselines across both Qwen3-8B and Qwen3-4B-Thinking-2507.
\end{abstract}

\vspace{-0.2in}
\section{Introduction}
\label{sec:introduction}

\vspace{-0.1in}
Scaling inference-time computation improves the performance of LLMs on complex reasoning and programming tasks~\citep{snell2024scaling, deepseekai2025deepseekr1incentivizingreasoningcapability}. The main approach for this is to generate multiple independent samples and select the best candidate via majority voting~\citep{wang2023h} or verification against test cases~\citep{jain2024livecodebench}. While such best-of-$K$ strategies reliably improve accuracy as $K$ increases, they scale inefficiently: each sample is generated without knowledge of prior attempts, leading to redundant exploration of similar trajectories. For problems that require long reasoning chains, where a single early error invalidates the entire generation, this redundancy is  wasteful.

\vspace{-0.04in}
Sequential multi-turn generation attempts to address this by incorporating environmental feedback into the context window~\citep{madaan2023self}. The model generates a subsequent attempt conditioned on the failure of its previous output. However, sequential refinement often stagnates: when faced with genuinely difficult reasoning challenges, models frequently repeat similar logical errors across turns or fail to identify the precise point where the logic diverged~\citep{kumar2024training}. The context window fills with unsuccessful attempts without yielding a correct solution.

\begin{figure*}[t]
    \centering
    \resizebox{\textwidth}{!}{%
    \begin{tikzpicture}[
        >={Stealth[length=2.5mm, width=1.5mm]},
        node distance=1.2cm and 1.5cm,
        prompt/.style={circle, draw=black, fill=black!80, inner sep=2.5pt, text=white, font=\small\bfseries},
        token/.style={circle, draw=DarkBlue, fill=DarkBlue!20, thick, inner sep=2pt},
        badtoken/.style={circle, draw=red!80!black, fill=red!20, thick, inner sep=2pt},
        goodtoken/.style={circle, draw=DarkGreen, fill=DarkGreen!20, thick, inner sep=2pt},
        terminalbad/.style={star, star points=4, star point ratio=2, draw=red!80!black, fill=red!80, inner sep=1.5pt},
        terminalgood/.style={regular polygon, regular polygon sides=4, draw=DarkGreen, fill=DarkGreen!80, inner sep=2pt},
        path/.style={->, thick, draw=DarkBlue!70},
        wastepath/.style={->, thick, draw=red!60, dashed},
        highlightpath/.style={->, ultra thick, draw=DarkGreen},
        feedback/.style={->, thick, draw=orange, dashed, bend right=25},
        eval/.style={font=\fontsize{7}{8}\selectfont, text=red!80!black},
        grouplabel/.style={font=\bfseries\small, align=center, text=black}
    ]

    \begin{scope}[shift={(0,0)}]
        \node[grouplabel, yshift=2cm] at (2.5,0) {(A) Independent Sampling\\\mdseries(BoN)};
        
        \foreach \y in {1, 0, -1} {
            \node[prompt] (P\y) at (0, \y) {};
            \node[token] (T1\y) at (1, \y) {};
            \node[token] (T2\y) at (2, \y) {};
            \node[token] (T3\y) at (3, \y) {};
            \node[token] (T4\y) at (4, \y) {};
            \node[terminalbad] (End\y) at (5, \y) {};
            
            \draw[path] (P\y) -- (T1\y);
            \draw[path] (T1\y) -- (T2\y);
            \draw[path] (T2\y) -- (T3\y);
            \draw[path] (T3\y) -- (T4\y);
            \draw[path] (T4\y) -- (End\y);
            
            \begin{pgfonlayer}{background}
                \draw[line width=10pt, red!10, rounded corners] 
                    ([xshift=-2pt]T1\y.west) -- ([xshift=2pt]End\y.east);
            \end{pgfonlayer}
        }
        \node[terminalgood] (EndS) at (5, -1) {};
        \node[goodtoken] (T4S) at (4, -1) {};
        \node[goodtoken] (T3S) at (3, -1) {};
        \node[goodtoken] (T2S) at (2, -1) {};
        \node[goodtoken] (T1S) at (1, -1) {};
        \begin{pgfonlayer}{background}
            \draw[line width=10pt, DarkGreen!10, rounded corners] 
                ([xshift=-2pt]T1S.west) -- ([xshift=2pt]EndS.east);
        \end{pgfonlayer}
        \draw[highlightpath] (P-1) -- (T1S);
        \draw[highlightpath] (T1S) -- (T2S);
        \draw[highlightpath] (T2S) -- (T3S);
        \draw[highlightpath] (T3S) -- (T4S);
        \draw[highlightpath] (T4S) -- (EndS);

        \node[font=\footnotesize, text=red!60!black, anchor=north] at (2.5, -1.8) {\bfseries Redundant compute};
    \end{scope}

    \begin{scope}[shift={(6,0)}]
        \node[grouplabel, yshift=2cm] at (2.5,0) {(B) Sequential\\\mdseries Refinement};
        
        \node[prompt] (P) at (0, 1) {};
        
        \node[token] (T11) at (1, 1) {};
        \node[token] (T12) at (2, 1) {};
        \node[token] (T13) at (3, 1) {};
        \node[token] (T14) at (4, 1) {};
        \node[terminalbad] (End1) at (5, 1) {};
        
        \draw[path] (P) -- (T11);
        \draw[path] (T11) -- (T12);
        \draw[path] (T12) -- (T13);
        \draw[path] (T13) -- (T14);
        \draw[path] (T14) -- (End1);

        \node[prompt] (Pf) at (0, 0) {};
        \node[token] (T21) at (1, 0) {};
        \node[token] (T22) at (2, 0) {};
        \node[token] (T23) at (3, 0) {};
        \node[token] (T24) at (4, 0) {};
        \node[terminalbad] (End2) at (5, 0) {};

        \draw[path] (Pf) -- (T21);
        \draw[path] (T21) -- (T22);
        \draw[path] (T22) -- (T23);
        \draw[path] (T23) -- (T24);
        \draw[path] (T24) -- (End2);

        \node[prompt] (Pf2) at (0, -1) {};
        \node[goodtoken] (T31) at (1, -1) {};
        \node[goodtoken] (T32) at (2, -1) {};
        \node[goodtoken] (T33) at (3, -1) {};
        \node[goodtoken] (T34) at (4, -1) {};
        \node[terminalgood] (End3) at (5, -1) {};

        \draw[highlightpath] (Pf2) -- (T31);
        \draw[highlightpath] (T31) -- (T32);
        \draw[highlightpath] (T32) -- (T33);
        \draw[highlightpath] (T33) -- (T34);
        \draw[highlightpath] (T34) -- (End3);

        \draw[feedback, rounded corners=4pt] (End1) -- (5.4, 1) -- (5.4, 0.5) -- (-0.4, 0.5) -- (-0.4, 0) -- (Pf);
        \node[font=\fontsize{6}{7}\selectfont, text=orange!90!black, fill=white, inner sep=1pt] at (2.5, 0.5) {Feedback / Reprompt};

        \draw[feedback, rounded corners=4pt] (End2) -- (5.4, 0) -- (5.4, -0.5) -- (-0.4, -0.5) -- (-0.4, -1) -- (Pf2);
        \node[font=\fontsize{6}{7}\selectfont, text=orange!90!black, fill=white, inner sep=1pt] at (2.5, -0.5) {Feedback / Reprompt};

        \begin{pgfonlayer}{background}
            \draw[line width=10pt, red!10, rounded corners] ([xshift=-2pt]T11.west) -- ([xshift=2pt]End1.east);
            \draw[line width=10pt, red!10, rounded corners] ([xshift=-2pt]T21.west) -- ([xshift=2pt]End2.east);
        \end{pgfonlayer}

        \node[font=\footnotesize, text=red!60!black, anchor=north] at (2.5, -1.8) {\bfseries Evaluates only at end};
    \end{scope}

    \begin{scope}[shift={(12,0)}]
        \node[grouplabel, yshift=2cm] at (2.5,0) {(C) {\ttsd} Prefix Tree\\\mdseries (Ours)};
        
        \node[prompt] (P) at (0, 0) {};
        
        \node[token] (T1) at (1, 0) {};
        \node[token] (T2) at (2, 0) {};
        
        \draw[path] (P) -- (T1);
        \draw[path] (T1) -- (T2);

        \node[badtoken, label={[eval, label distance=1pt]below:\textbf{\texttimes} Error}] (T3) at (3, -1) {};
        \node[token] (T4) at (4, -1) {};
        \node[terminalbad] (End1) at (5, -1) {};
        
        \draw[path] (T2) to[out=0, in=180] (T3);
        \draw[wastepath] (T3) -- (T4) node[midway, below=-2pt] {\tiny prune};
        \draw[wastepath] (T4) -- (End1);

        \node[token] (T3b) at (3, 0) {};
        \node[badtoken, label={[eval, label distance=1pt]above:\textbf{\texttimes} Error}] (T4b) at (4, 0.5) {};
        \node[terminalbad] (End2) at (5, 0.5) {};

        \draw[path] (T2) -- (T3b);
        \draw[path] (T3b) to[out=0, in=180] (T4b);
        \draw[wastepath] (T4b) -- (End2);

        \node[goodtoken] (T4c) at (4, -0.5) {};
        \node[terminalgood] (End3) at (5, -0.5) {};

        \draw[highlightpath] (T3b) to[out=0, in=180] (T4c);
        \draw[highlightpath] (T4c) -- (End3);

        \draw[decorate, decoration={brace, amplitude=5pt}, thick, DarkBlue] 
            (2, 0.3) -- (1, 0.3) node[midway, above=5pt, font=\footnotesize] {Shared Prefix};
            
        \node[font=\footnotesize, text=DarkGreen, anchor=north] at (2.5, -1.8) {\bfseries Localized error branching};
    \end{scope}

    \end{tikzpicture}%
    }
    \caption{\textbf{Search Tree Typology Comparison.} (A) Independent Sampling generates entirely separate full-length trajectories, leading to high computational waste from parallel exploration of similar paths. (B) Sequential Refinement methods link full trajectories end-to-end, evaluating only at terminal steps. (C) \ttsd{} structures search as a prefix-sharing tree, explicitly capitalizing on token-level error localization to branch exactly where reasoning diverges, avoiding redundant computation of shared prefixes and averting terminal failure.}\label{fig:typology}
\end{figure*}
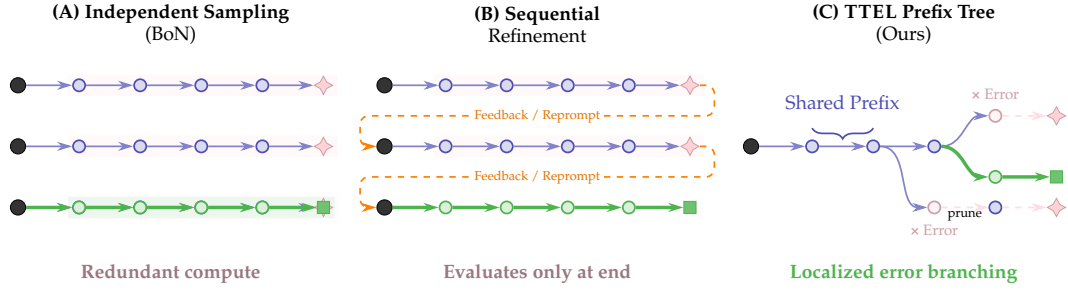

A deeper issue is that post-hoc feedback is inherently \emph{trajectory-conditional}. When a candidate solution fails, the resulting feedback, such as an environment generated error, or self-generated critique, is informative about the specific reasoning path that produced that failure. If a model is restarted from scratch, it may never revisit the same intermediate states at which the feedback is most relevant. 
This creates a fundamental mismatch: the supervision is path-specific, but existing inference-time strategies such as sequential refinement largely apply it as though it were a global revision signal.

Tree-structured search methods, such as Tree of Thoughts~\citep{yao2023tree}, offer an alternative by exploring multiple reasoning paths and backtracking. These methods typically operate at the level of coarse reasoning steps, relying on self-evaluation heuristics to prune branches. Population-based approaches like Recursive Self-Aggregation~\citep{venkatraman2025recursive} refine solutions by aggregating information across a pool of candidates. Both families of methods improve over flat sampling, but neither leverages fine-grained, token-level signals about where reasoning went wrong. Essentially, these existing approaches  either lose the state on which feedback is informative, or preserve the state without extracting a precise corrective signal from it.

\vspace{-0.03in}
This suggests a more targeted view of inference-time improvement: rather than treating feedback as a generic instruction for a fresh retry, can a model use feedback to localize error \emph{on the specific trajectory that produced it}, retain the useful prefix of that trajectory, and selectively regenerate only the suffix where the reasoning becomes suspect? Motivated by this question, we investigate the following core problem:
\begin{tcolorbox}[colback=gray!5!white, colframe=gray!40!white, arc=1mm, boxrule=0.5pt, left=2mm, right=2mm, top=0mm, bottom=0mm, before skip=0.5em, after skip=0.5em]
\centering \textit{Can an LLM perform token-level credit assignment on its own failed trajectories, \\ and leverage this localized signal to guide test-time scaling optimally?}
\end{tcolorbox}

We propose {\ttsd}, an inference-time search algorithm that performs token-level error localization without gradient updates. When a generated solution fails, we prompt the model with the environmental feedback to re-evaluate the original trajectory via prompt log-probabilities. By comparing token probabilities before and after observing this feedback, we identify positions of maximal disagreement. A sharp drop in probability indicates that the model strongly rejects its prior decision at that specific step. 

\vspace{-0.03in}
To ensure these probability shifts are genuinely driven by the feedback rather than generic context changes, we filter the signal against a baseline with non-diagnostic feedback. {\ttsd} then uses this localized signal to truncate the reasoning trace at the highest-confidence error and branch a new generation from that point. By directing computation toward correcting specific mistakes rather than regenerating entire solutions, this feedback-guided tree search achieves substantially higher token efficiency than standard sampling baselines across diverse reasoning tasks, including competitive programming~\citep{jain2024livecodebench} and complex mathematical reasoning ~\citep{balunovic2025matharena}.

\vspace{-0.12in}
\paragraph{Our Contributions.} Primarily, this paper shows that LLMs can error correct failing trajectories at inference time based on feedback, and use it to guide test-time search. Specifically:

\vspace{-0.08in}
\textbf{1. Token-Level Error Localization:} We introduce {\ttsd}, an inference-time algorithm that uses probability contrasts by comparing informed environmental feedback against non-diagnostic feedback to localize errors during generation.

\vspace{-0.08in}
\textbf{2. Efficient Test-Time Search:} By coupling local credit assignment with trajectory retention, {\ttsd} truncates and branches only the faulty suffix of a trajectory. This maximizes the reuse of valid reasoning prefixes and eliminates the redundant exploration.

\vspace{-0.08in}
\textbf{3. Empirical Efficiency:} We establish strictly dominating Pareto frontiers across math benchmarks like AIME-2025, and coding benchmarks like LiveCodeBench where \ttsd\ attains pass@64 of 71.0\% at roughly half the token cost of independent sampling.

\section{Related Work}
\label{sec:related_work}

We provide an abbreviated overview of related work below; an extended discussion mapping {\ttsd} to the broader literature is available in Appendix~\ref{app:related_work_extended}.

\paragraph{Test-Time Scaling and Tree Search.}
Reasoning performance scales reliably with inference-time compute \citep{snell2024scaling, deepseekai2025deepseekr1incentivizingreasoningcapability}, typically through robust independent sampling arrays \citep{wang2023h}. To iterate beyond flat sampling, sequential refinement methods prepend execution feedback to prompt self-correction \citep{madaan2023self, klissarov2026improving}. However, because they treat trajectory-specific feedback as a global instruction, these models often regenerate reasoning from scratch, triggering context bloat or hallucinated logic \citep{kumar2024training}. Structured backtracking algorithms like Tree of Thoughts \citep{yao2023tree} and population-based search \citep{venkatraman2025recursive} remedy this by exploring partial branches, but they rely on programmatic chunking and external heuristics rather than internal, token-level signals.

\paragraph{On-Policy Distillation and Preference Optimization.}
Recent work has increasingly utilized natural language assessment for dense credit assignment. By evaluating the ratio between a base policy and a feedback-conditioned teacher, methods like Generalized Knowledge Distillation \citep{agarwal2024policy} and On-Policy Self-Distillation \citep{zhao2026self} optimize student models directly on their own generated rollouts. This core mathematical insight, minimizing the Kullback-Leibler divergence driven by $\pi_\theta(\cdot \mid x) / \pi_\theta(\cdot \mid x, f)$, has established the foundation for new reinforcement learning algorithms, most notably Self-Distillation Policy Optimization (SDPO; \citealt{hubotter2026reinforcement}) and RLTF \citep{song2026expanding}. Similar mechanisms have also been applied successfully in experiential context distillation (OPCD; \citealt{ye2026policy}) and continual learning (SDFT; \citealt{shenfeld2026self}).

\paragraph{Repurposing the Signal for Inference-Time Search.}
Crucially, this existing body of work restricts the application of the $\pi_\theta(\cdot \mid x) / \pi_\theta(\cdot \mid x, f)$ divergence to offline gradient updates, yielding static policy networks. The direct capability of this ratio to localize errors purely as an online, inference-time heuristic has remained fundamentally unaddressed. {\ttsd} bridges this gap by repurposing the token-level signal used in the SDPO framework. Specifically, {\ttsd} extracts an organic, native pruning criterion that structures and bounds test-time search without requiring parameter updates.

\section{Preliminaries and Problem Formulation}
\label{sec:preliminaries}

Let $\pi_\theta$ denote a pre-trained autoregressive language model over vocabulary $\mathcal V$. Given a problem prompt $x$, the model generates a trajectory
$
y = (y_1,\dots,y_N)
$
with token probabilities
\begin{equation}
p_t^{(S)} := \pi_\theta(y_t \mid x, y_{<t}), \qquad y_{<t} := (y_1,\dots,y_{t-1}).
\end{equation}
For a realized trajectory $y$, we refer to $p_t^{(S)}$ as the \emph{student token probability}, i.e., the probability assigned to the generated token $y_t$ during the original rollout.

We assume access to an execution environment $\mathcal E$, which evaluates a candidate trajectory $y$ on problem $x$ and returns a textual feedback string $f$.

The central challenge in test-time scaling is a \emph{credit assignment bottleneck}: when a generated trajectory fails, the environment evaluates the trajectory globally, but does not identify the specific token position at which the reasoning first became inconsistent with a successful solution. Standard best-of-$K$ sampling discards failed trajectories entirely, while sequential refinement conditions on feedback but does not explicitly localize where to revise the trajectory. Our goal is to use the model itself to perform token-level post-hoc error localization on a failed trajectory, and then use that localized signal to guide branching during inference-time search.

\clearpage
\section{Methodology}
\label{sec:methodology}

{\ttsd} operates on a single pre-trained language model $\pi_\theta$ that serves as both the generator and the evaluator of its own outputs. The core mechanism computes the same token-level divergence signal that train-time self-distillation algorithms use to update model weights~\citep{hubotter2026reinforcement}, but instead of aggregating this signal into a gradient step, {\ttsd} uses it to prune and branch an inference-time search tree. We formalize the algorithm in two core components: detecting and filtering token-level error signals (spikes), and utilizing those signals to guide a search and branching strategy over candidate solutions.

\subsection{Spike Detection and Filtering}
\label{sec:spike_detection}

\begin{wrapfigure}{R}{0.55\textwidth}
    \vspace{-10pt}
    \centering
    \begin{tikzpicture}[
        font=\sffamily,
        >={Stealth[length=2mm, width=1.5mm]}
    ]
    
    \usepgfplotslibrary{groupplots}
    
    \begin{groupplot}[
        group style={
            group size=1 by 3,
            vertical sep=1cm,
        },
        width=7.5cm, 
        height=3.0cm,
        xmin=0, xmax=21,
        ymin=0, ymax=5,
        xtick=\empty,
        ytick=\empty,
        axis x line=bottom,
        axis y line=left,
        enlarge x limits=false,
        clip=false
    ]
    
    \nextgroupplot[ylabel={Raw Spike $\Delta_t$}, ylabel style={align=center}]
        \addplot[ycomb, color=blue!80, line width=1.5pt, mark=*] coordinates {
            (2, 0.8) (3, 0.4) (5, 4.5) (7, 1.2) (9, 0.5) (11, 4.2) (14, 3.2) (16, 0.9) (18, 2.5) (20, 1.1)
        };
        \draw[red, dashed, thick] (axis cs:0, 2.5) -- (axis cs:21, 2.5) node[right, font=\footnotesize] {$\tau$};
        \node[anchor=south west, font=\footnotesize] at (axis cs: 0, 5) {\textbf{Step 1:} Compute Teacher Divergence};

    \nextgroupplot[ylabel={Noise $\Delta_t^{(B)}$}, ylabel style={align=center}]
        \addplot[ycomb, color=orange!80, line width=1.5pt, mark=*] coordinates {
            (2, 0.7) (3, 0.3) (5, 0.6) (7, 1.0) (9, 0.4) (11, 4.0) (14, 0.8) (16, 0.5) (18, 0.2) (20, 0.9)
        };
        \draw[red, dashed, thick] (axis cs:0, 2.0) -- (axis cs:21, 2.0) node[right, font=\footnotesize] {$\tau_B$};
        \node[anchor=south west, font=\footnotesize] at (axis cs: 0, 5) {\textbf{Step 2:} Compute Baseline Divergence};
        
        \draw[->, DarkBlue, shorten >=2pt] (axis cs: 12, 4.5) node[right, font=\scriptsize, text=DarkBlue] {Context-level artifact (Noise)} to[out=180, in=90] (axis cs: 11, 4.1);

    \nextgroupplot[ylabel={Filtered $E_{\text{filtered}}$}, ylabel style={align=center}, 
                   xlabel={Token Position (Time Step $t$)}, 
                   xtick={5, 14, 18}, xticklabels={$t_1^*$, $t_B$, $t_C$}]
        \addplot[ycomb, color=DarkGreen, line width=1.5pt, mark=*] coordinates {
            (5, 4.5) (14, 3.2) (18, 2.5)
        };
        \addplot[ycomb, color=red!80, line width=1.8pt, mark=star, mark size=4pt] coordinates {
            (5, 4.5)
        };
        \node[anchor=south west, font=\footnotesize] at (axis cs: 0, 5) {\textbf{Step 3:} Filter \& Identify Branch Point $t_1^*$};
        
    \end{groupplot}

    \end{tikzpicture}
    \caption{\textbf{Token-Level Error Localization.} {\ttsd} filters raw feedback-induced probability drops (blue) against a non-diagnostic context baseline (orange) to distinguish genuine errors from context artifacts. The maximal filtered spike (green) is selected as the precise branch point $t_1^*$.}\label{fig:spike_mechanism}
\end{wrapfigure}
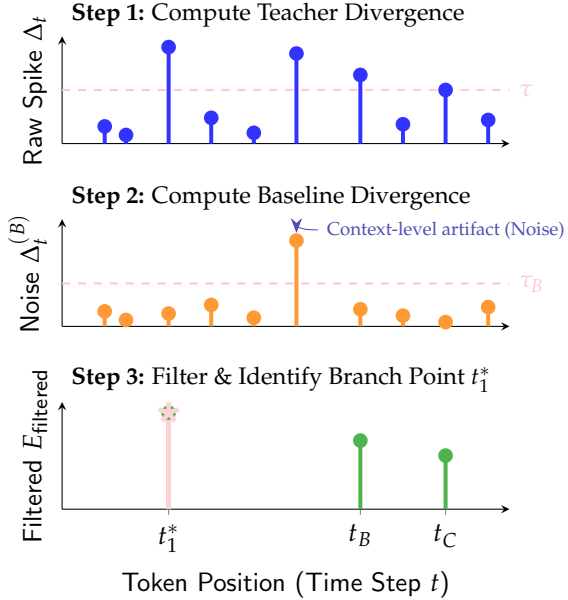

Given a prompt $x$, the model first generates a candidate trajectory
$
y = (y_1,\dots,y_N)
$
under standard autoregressive decoding. For this realized trajectory, we record the \emph{student} token probabilities
\begin{equation}
p_t^{(S)} := \pi_\theta(y_t \mid x, y_{<t}).
\end{equation}

We obtain textual feedback $f$ and then re-score the \emph{same} token sequence $y$ under a feedback-augmented context, yielding the \emph{teacher} token probabilities
\begin{equation}
p_t^{(T)}(f) := \pi_\theta(y_t \mid x, y_{<t}, f).
\end{equation}


\textbf{Note:} For notational simplicity, $\pi_\theta(y_t \mid x, y_{<t}, f)$ denotes the probability assigned to token $y_t$ when the original trajectory is re-scored under feedback $f$. Here, $y_{<t}$ is the original trajectory prefix preceding $y_t$.

We define the raw feedback-conditioned spike at token $t$ as
\begin{equation}
\Delta_t(f) := p_t^{(S)} - p_t^{(T)}(f).
\label{eq:raw_spike}
\end{equation}
A large positive value of $\Delta_t(f)$ indicates that, after observing feedback $f$, the model assigns substantially lower probability to its original token choice $y_t$ than it did during generation.

However, this raw spike can reflect two distinct effects: (i) genuine semantic re-evaluation induced by task-relevant feedback, and (ii) generic probability shifts caused by appending additional text to the context window. To control for the latter, we introduce a null feedback string $f_{\varnothing}$ that contains a non-diagnostic instruction such as ``Repeat the previous attempt word by word, but skip portions of redundant thinking''. We then compute the corresponding \emph{baseline} token probabilities

\begin{equation}
p_t^{(B)}(f_{\varnothing}) := \pi_\theta(y_t \mid x, y_{<t}, f_{\varnothing}).
\end{equation}
This yields the baseline spike
\begin{equation}
\Delta_t(f_{\varnothing}) := p_t^{(S)} - p_t^{(B)}(f_{\varnothing}).
\label{eq:baseline_spike}
\end{equation}

We then define the \emph{filtered} spike score
\begin{equation}
g_t := \Delta_t(f) - \Delta_t(f_{\varnothing})
     = p_t^{(B)}(f_{\varnothing}) - p_t^{(T)}(f).
\label{eq:filtered_spike}
\end{equation}
This score isolates feedback-specific disagreement by subtracting out generic context-induced shifts that also arise under non-diagnostic feedback.

In practice, we retain token positions whose teacher spike exceeds a threshold while the null-feedback spike remains small:
\begin{equation}
E_{\mathrm{filtered}}
=
\left\{
t \in \{1,\dots,N\}
\;\middle|\;
\Delta_t(f) > \tau
\;\land\;
\Delta_t(f_{\varnothing}) \le \tau_B
\right\}.
\label{eq:filtered_error_set}
\end{equation}

\subsection{Search and Branching Strategy}
\label{sec:branching}

\begin{wrapfigure}{R}{0.55\textwidth}
    \vspace{-35pt}
    \centering
    \resizebox{\linewidth}{!}{%
    \begin{tikzpicture}[
        font=\sffamily\normalsize,
        >={Stealth[length=2mm, width=1.5mm]},
        prompt/.style={rectangle, rounded corners=2pt, draw=titlecolor!80!black, fill=titlecolor, text=white, inner sep=4pt, font=\bfseries},
        token/.style={circle, draw=DarkBlue!80, fill=DarkBlue!20, inner sep=1.8pt, minimum size=0.4cm},
        reused/.style={circle, draw=DarkGreen!80, fill=DarkGreen!30, inner sep=1.8pt, minimum size=0.4cm},
        deadtoken/.style={circle, draw=gray!60, fill=gray!20, inner sep=1.8pt, minimum size=0.4cm},
        path/.style={->, thick, draw=DarkBlue!50},
        reusedpath/.style={->, ultra thick, draw=DarkGreen!80},
        deadpath/.style={->, thick, draw=gray!40},
        spike/.style={thick, draw=blue!70},
        maxspike/.style={ultra thick, draw=red!80, cap=round},
        evalbad/.style={text=red!80, font=\Large\bfseries},
        evalgood/.style={text=DarkGreen, font=\Large\bfseries}
    ]

    \def\ydelta{-3.2} 
    \def\xstep{0.55}  

    \begin{scope}[shift={(0, 0)}]
        \node[anchor=east, font=\large\bfseries] at (-2.0, 0) {Iteration 1};
        \node[prompt] (P1) at (-1.0, 0) {Prompt $x$};

        \draw[gray!30, thick] (0, 0) -- (12*\xstep, 0);
        \draw[red!40, dashed, thick] (0, 0.8) -- (12*\xstep, 0.8) node[right, font=\normalsize, text=red!80] {$\tau$};

        \foreach \h [count=\xi] in {0.2, 0.4, 0.1, 0.5, 0.3, 2.0, 0.6, 0.2, 0.9, 0.4, 0.7, 0.1} {
            \pgfmathsetmacro{\x}{(\xi-1)*\xstep}
            \node[token] (T1-\xi) at (\x, 0) {};
            \ifdim\h pt>1.9pt
                \draw[maxspike] (T1-\xi.north) -- (\x, \h) node[above, font=\normalsize, text=red!80!black] {$t_1^*$};
                \node[star, star points=5, star point ratio=2.25, fill=red!70, inner sep=1pt] at (\x, \h) {};
            \else
                \draw[spike] (T1-\xi.north) -- (\x, \h);
            \fi
        }

        \draw[path] (P1) -- (T1-1);
        \foreach \xi [count=\xiprev] in {2,3,4,5,6,7,8,9,10,11,12} {
             \draw[path] (T1-\xiprev) -- (T1-\xi);
        }

        \node[evalbad] at (12.5*\xstep, 0) {\texttimes};
        
        \node[anchor=north west, font=\large, text=black!80] at (0, -0.6) {Spike $t_1^*$ found. Branch from prefix $y_{<t_1^*}$.};
    \end{scope}

    \begin{scope}[shift={(0, \ydelta)}]
        \node[anchor=east, font=\large\bfseries] at (-2.0, 0) {Iteration 2};
        \node[prompt] (P2) at (-1.0, 0) {Prompt $x$};

        \draw[red!40, dashed, thick] (0, 0.8) -- (12*\xstep, 0.8);

        \foreach \i in {1,...,5} {
            \pgfmathsetmacro{\x}{(\i-1)*\xstep}
            \node[reused] (T2-\i) at (\x, 0) {};
            \draw[spike, opacity=0.3] (T2-\i.north) -- (\x, 0.2); 
        }

        \foreach \h [count=\i from 6] in {0.2, 0.5, 0.4, 1.6, 0.3, 1.0, 0.6} {
            \pgfmathsetmacro{\x}{(\i-1)*\xstep}
            \node[token] (T2-\i) at (\x, 0) {};
            \ifdim\h pt>1.5pt
                \draw[maxspike] (T2-\i.north) -- (\x, \h) node[above, font=\normalsize, text=red!80!black] {$t_2^*$};
                \node[star, star points=5, star point ratio=2.25, fill=red!70, inner sep=1pt] at (\x, \h) {};
            \else
                \draw[spike] (T2-\i.north) -- (\x, \h);
            \fi
        }

        \foreach \i in {6,7,8,9} {
            \pgfmathsetmacro{\x}{(\i-1)*\xstep}
            \node[deadtoken, opacity=0.3] at (\x, 0.5) {};
        }

        \draw[reusedpath] (P2) -- (T2-1);
        \foreach \i [count=\iprev] in {2,3,4,5} {
             \draw[reusedpath] (T2-\iprev) -- (T2-\i);
        }
        
        \foreach \i [count=\iprev from 5] in {6,7,8,9,10,11,12} {
             \draw[path] (T2-\iprev) -- (T2-\i);
        }

        \node[evalbad] at (12.5*\xstep, 0) {\texttimes};
        
        \pgfmathsetmacro{\bracex}{4*\xstep}
        \draw[decorate, decoration={brace, amplitude=4pt, mirror}, thick, DarkGreen] (0, -0.3) -- (\bracex, -0.3) node[midway, below=4pt, font=\normalsize] {Reused Prefix};
        
        \node[anchor=north west, font=\large, text=black!80] at (0, -1.2) {Spike $t_2^*$ found. Branch from prefix $y_{<t_2^*}$.};
    \end{scope}

    \begin{scope}[shift={(0, 2*\ydelta)}]
        \node[anchor=east, font=\large\bfseries] at (-2.0, 0) {Iteration 3};
        \node[prompt] (P3) at (-1.0, 0) {Prompt $x$};

        \draw[red!40, dashed, thick] (0, 0.8) -- (14*\xstep, 0.8);

        \foreach \i in {1,...,8} {
            \pgfmathsetmacro{\x}{(\i-1)*\xstep}
            \node[reused] (T3-\i) at (\x, 0) {};
            \draw[spike, opacity=0.3] (T3-\i.north) -- (\x, 0.2); 
        }

        \foreach \h [count=\i from 9] in {0.3, 0.5, 0.2, 0.7, 0.4, 0.6} {
            \pgfmathsetmacro{\x}{(\i-1)*\xstep}
            \node[token] (T3-\i) at (\x, 0) {};
            \draw[spike] (T3-\i.north) -- (\x, \h);
        }

        \draw[reusedpath] (P3) -- (T3-1);
        \foreach \i [count=\iprev] in {2,3,4,5,6,7,8} {
             \draw[reusedpath] (T3-\iprev) -- (T3-\i);
        }
        
        \foreach \i [count=\iprev from 8] in {9,10,11,12,13,14} {
             \draw[path] (T3-\iprev) -- (T3-\i);
        }

        \node[evalbad] at (14.5*\xstep, 0) {\texttimes};
        
        \pgfmathsetmacro{\bracex}{7*\xstep}
        \draw[decorate, decoration={brace, amplitude=4pt, mirror}, thick, DarkGreen] (0, -0.3) -- (\bracex, -0.3) node[midway, below=4pt, font=\normalsize] {Expanded Reused Prefix};
        
        \node[anchor=north west, font=\large, text=black!80] at (0, -1.2) {\textbf{No spike} crosses $\tau$. Trigger \textbf{Restart}.};
    \end{scope}

    \begin{scope}[shift={(0, 3*\ydelta)}]
        \node[anchor=east, font=\large\bfseries] at (-2.0, 0) {Iteration 4};
        \node[prompt] (P4) at (-1.0, 0) {Prompt $x$};

        \draw[red!40, dashed, thick] (0, 0.8) -- (12*\xstep, 0.8);

        \foreach \h [count=\i] in {0.2, 0.5, 0.1, 0.3, 0.4, 0.1, 0.2, 0.7, 0.4, 0.3, 0.5} {
            \pgfmathsetmacro{\x}{(\i-1)*\xstep}
            \node[token] (T4-\i) at (\x, 0) {};
            \draw[spike] (T4-\i.north) -- (\x, \h);
        }

        \draw[path] (P4) -- (T4-1);
        \foreach \i [count=\iprev] in {2,3,4,5,6,7,8,9,10,11} {
             \draw[path] (T4-\iprev) -- (T4-\i);
        }

        \node[evalgood] at (11*\xstep, 0) {\checkmark};
        
        \node[anchor=north west, font=\large, text=black!80] at (0, -0.6) {Fresh trajectory.};
    \end{scope}

    \end{tikzpicture}%
    }
    \caption{\textbf{{\ttsd} Search and Branching Dynamics.} The reasoning trajectory iteratively evolves to isolate errors. At Iteration 1, the model detects a filtered error spike $t_1^*$ (red) and branches a new generation from the retained prefix $y_{<t_1^*}$ (green). If a trajectory fails without a discernible spike (e.g., Iteration 3 where $E_{\text{filtered}} = \emptyset$), the tree search triggers a restart from the root prompt.\vspace{-30pt}}\label{fig:search_branching}
\end{wrapfigure}
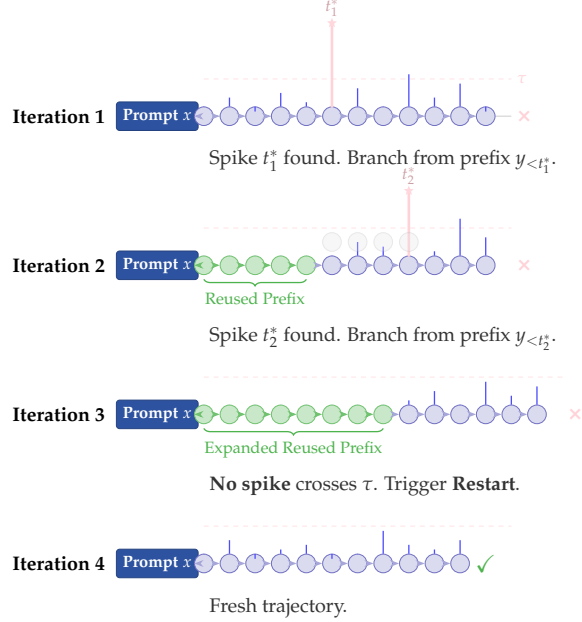

Once the localized error set $E_{\mathrm{filtered}}$ is computed, the algorithm must choose where to branch. We select the branch point as the token with the strongest feedback-specific disagreement:

\begin{wrapfigure}{R}{0.52\textwidth} 
\vspace{-30pt} 
\begin{tcolorbox}[
    enhanced,
    width=\linewidth, 
    colback=white,
    colframe=algblue,
    coltitle=white,
    fonttitle=\bfseries\small,
    title=Alg. 1 {\ttsd}: Feedback-Guided Tree Search,
    boxrule=0.6pt,
    arc=1.5mm,
    left=1mm, right=1mm, top=1mm, bottom=1mm, 
]
\begin{algorithmic}[1]
\footnotesize 
\REQUIRE Prompt $x$, model $\pi_\theta$, env. $\mathcal{E}$, thresholds $\tau,\tau_B$, budget $B$
\STATE $\mathcal{T} \leftarrow \{\emptyset\}$ \hfill \textcolor{alggreen}{\texttt{// retained prefixes}}
\STATE $\mathcal{S} \leftarrow \emptyset$ \hfill \textcolor{alggreen}
{\texttt{// generated candidates}}
\STATE $f \leftarrow \emptyset$ \hfill \textcolor{alggreen}
{\texttt{// no feedback initially}}
\\
\dg{budget should also include number of rollouts in the next line?}
\WHILE{tokens generated $< B$}
    \STATE Select leaf $y_{<t} \in \mathcal{T}$
    \STATE $y \leftarrow \textsc{Generate}(\pi_\theta, x, y_{<t}, f)$
    \STATE $f \leftarrow \mathcal{E}(x, y)$ \hfill \textcolor{alggreen}{\texttt{// feedback}}
    \STATE \textcolor{algblue}{$\triangleright$\; \textit{Spike Detection}}
    \STATE $\Delta_t(f) \leftarrow p_t^{(S)} - p_t^{(T)}(f)$ \hfill \textcolor{alggreen}{\texttt{// cond. spike}}
    \STATE $\Delta_t(f_{\varnothing}) \leftarrow p_t^{(S)} - p_t^{(B)}(f_{\varnothing})$ \hfill \textcolor{alggreen}{\texttt{// null spike}}
    \STATE $E \leftarrow \{\, t \mid \Delta_t(f) > \tau \land \Delta_t(f_{\varnothing}) \le \tau_B \,\}$
    \STATE \textcolor{algblue}{$\triangleright$\; \textit{Branching}}
    \IF{$E \neq \emptyset$}
        \STATE $t^* \leftarrow \arg\max_{t \in E} (\Delta_t(f) - \Delta_t(f_{\varnothing}))$ \hfill 
        \STATE Add $y_{<t^*}$ to $\mathcal{T}$ \hfill \textcolor{algorange}{\texttt{// branch at max spike}}
    \ELSE
        \STATE $\mathcal{T} \leftarrow \{\emptyset\}$
        \STATE $f \leftarrow \emptyset$
\ENDIF
        \STATE $\mathcal{S} \leftarrow \mathcal{S} \cup \{y\}$   
\ENDWHILE
\STATE \textbf{return} $\mathcal{S}$
\end{algorithmic}
\end{tcolorbox}
\vspace{-55pt} 
\end{wrapfigure}

\begin{equation}
t^*
=
\arg\max_{t \in E_{\mathrm{filtered}}} g_t.
\label{eq:branch}
\end{equation}

The search tree then truncates the failed trajectory at position $t^*-1$ and launches a new generation from the retained prefix $y_{<t^*}$. This mechanism reuses the portion of the trajectory that precedes the strongest localized error signal, rather than discarding the full generation and restarting from scratch.

When $E_{\mathrm{filtered}} = \varnothing$, the available feedback does not localize any actionable token-level error. In this case, the algorithm performs a restart from the original prompt $x$.

\section{The TTEL Algorithm}
\label{sec:algorithm}

Algorithm~1 summarizes the full {\ttsd} procedure. The search tree $\mathcal{T}$ stores \emph{retained prefixes} of previously generated trajectories, with the root corresponding to the empty prefix. At each iteration, the algorithm selects a leaf prefix $y_{<t}$ from $\mathcal{T}$ and generates a full candidate continuation conditioned on both the original problem prompt $x$ and the retained prefix. The resulting trajectory is then evaluated by the execution environment $\mathcal{E}$, which returns feedback $f$ for re-scoring and subsequent regeneration. The diagnostic component of $f$ may contain environment-derived information, such as a compiler error, runtime error, or failing public test case.


The algorithm performs token-level spike detection by re-scoring the same trajectory under two alternative contexts: a feedback-conditioned context using the actual feedback $f$, and a null-feedback context using a non-diagnostic message $f_{\varnothing}$. This yields a feedback-conditioned spike score $\Delta_t(f) = p_t^{(S)} - p_t^{(T)}(f)$ and a baseline spike score $\Delta_t(f_{\varnothing}) = p_t^{(S)} - p_t^{(B)}(f_{\varnothing})$ at each token position. The filtered spike score \rsc{5. small s or delta? Also check S. Should null-baseline input be shown}
\[
g_t = \Delta_t(f) - \Delta_t(f_{\varnothing})
\]
captures disagreement that is specific to the observed failure signal rather than generic context augmentation.

The algorithm then constructs the localized error set
\[
E = \{\, t \mid \Delta_t(f) > \tau \;\land\; \Delta_t(f_{\varnothing}) \le \tau_B \,\},
\]
which contains token positions whose probability drops sharply under true feedback but not under null feedback. If this set is nonempty, {\ttsd} selects the branch point$t^* = \arg\max_{t \in E} g_t$
and adds the retained prefix $y_{<t^*}$ to the tree. A subsequent generation from this prefix reuses the earlier portion of the trajectory while revising the region where the model exhibits the strongest feedback-specific disagreement. If no actionable spike is detected, the algorithm performs a restart by adding the empty prefix back into the tree, thereby encouraging continued exploration from scratch rather than overcommitting to an uninformative failed trajectory.

The procedure repeats until the inference-time budget $K$ is exhausted, where the budget may be measured in generated tokens, completed trajectories, or environment calls. The procedure returns the generated candidate set $\mathcal S$, which is evaluated using pass@$k$. An external selector, such as a verifier score, execution outcome, or majority-based aggregation, may optionally be applied when a single answer is required. In this way, {\ttsd} converts token-level self-re-evaluation into an explicit tree-search policy that adaptively allocates computation toward revising localized errors instead of repeatedly regenerating entire solutions.


\section{Theoretical Analysis}
\label{sec:theory}

\paragraph{Notation.}
Let $x$ be the problem, $y$ a failed trajectory, and $f$ the corresponding feedback, which includes $y$. Let $c = (x, f)$ denote this shared context. Suppose {\ttsd} selects a branch point $t^*$ and retains the anchor prefix $a = y_{<t^*}$.

We compare two post-failure strategies:

\vspace{-0.1in}
\paragraph{1. Restart:} Generating a fresh attempt from the context without anchoring, $Y^{\mathrm R} \sim \pi_\theta(\cdot \mid c)$.


\vspace{-0.1in}
\paragraph{2. Branch:} Forcing generation to begin from the anchor prefix, $Y^{\mathrm B} \sim \pi_\theta(\cdot \mid c, a)$.


Let $V(\hat{y}) \in \{0, 1\}$ denote whether candidate $\hat{y}$ correctly solves the problem.

\begin{theorem}[Branching Advantage]
\label{thm:branch_advantage}
Assume \emph{autoregressive prefix consistency}, meaning that freely generating a sequence that begins with $a$ is distributionally equivalent to explicitly anchoring generation at $a$:
\[
\mathcal{L}\bigl(Y^{\mathrm R}_{\ge |a|+1} \bigm| Y^{\mathrm R}_{< |a|+1} = a,\, c\bigr) = \mathcal{L}\bigl(Y^{\mathrm B}_{\ge |a|+1} \bigm| c, a\bigr).
\]

Let $p_{a} = \Pr(Y^{\mathrm R} \text{ begins with } a \mid c)$ be the probability that a restart naturally recovers the anchor prefix. Let $s_{a}$ and $s_{\neg a}$ be the probability of success given that the generation does, or does not, begin with $a$. Then:
\[
\Pr\bigl(V(Y^{\mathrm B}) \!=\! 1 \mid c, a\bigr) - \Pr\bigl(V(Y^{\mathrm R}) \!=\! 1 \mid c\bigr) \;=\; (1 - p_{a})(s_{a} - s_{\neg a}).
\]
\end{theorem}

\begin{proof}
By prefix consistency, the success rate of branching is exactly the success rate of restart conditioned on recovering the prefix: $\Pr(V(Y^{\mathrm B}) \!=\! 1) = s_{a}$. By the law of total probability, the unconditional success rate of restart is:
\[
\Pr(V(Y^{\mathrm R}) \!=\! 1) = p_{a} s_{a} + (1 - p_{a}) s_{\neg a}.
\]
Subtracting the second equation from the first yields the result.
\end{proof}

The theorem formally justifies the {\ttsd} branching mechanism over standard sequential refinement. Verifier feedback $f$ is trajectory-conditional: it provides correction specifically relevant to the reasoning state at the anchor prefix $a$. A restart strategy generated from scratch ($Y^{\mathrm R}$) may never naturally return to this necessary intermediate state ($p_a \ll 1$), rendering the feedback useless or misleading. 
\begin{tcolorbox}[
    enhanced,
    colback=gray!5!white,
    colframe=gray!55!black,
    boxrule=0.5pt,
    arc=1mm,
    left=2mm, right=2mm, top=1mm, bottom=1mm,
    title=\textbf{Takeaway}
]
Branching structurally guarantees that the generation is anchored exactly where the feedback's correction becomes actionable. It strictly improves over sequential restart as long as success is more likely once the correct prefix is reached ($s_a > s_{\neg a}$).
\end{tcolorbox}

\section{Experiments}
\label{sec:experiments}

In this section, we outline our experimental framework designed to evaluate the effectiveness of our proposed method, {\ttsd}. We focus our evaluation on two challenging domains demanding multi-step planning: complex mathematical reasoning and code generation. Our empirical evaluation is designed to answer the following core research questions (RQ):

\begin{itemize}[leftmargin=*, noitemsep, topsep=0pt]
    \item \textbf{RQ1 (Compute Scalability \& Generalization):} How efficiently does {\ttsd} scale inference-time compute to uncover correct trajectories compared to standard multi-turn paradigms, and does this scaling advantage generalize across environments with and without explicit verification feedback?
    \item \textbf{RQ2 (The Role of Context in Error Localization):} How strictly does accurate token-level credit assignment depend on the availability of rich environment execution signals and the retention of the historical reasoning trace?
    \item \textbf{RQ3 (Filtering Noise vs.\ Semantic Signal):} Is the null-feedback baseline subtraction strictly necessary to isolate genuine, feedback-driven model corrections from generic context-induced probability shifts?
\end{itemize}

\paragraph{Evaluation Benchmarks}
To rigorously assess the capabilities of our approach, we focus on the following two benchmark categories:
\begin{itemize}[leftmargin=*, noitemsep, topsep=0pt]
    \item \textbf{Mathematical Reasoning:} We evaluate on \textbf{AIME-25} dataset, and \textbf{HMMT-25} from the MathArena benchmark \citep{balunovic2025matharena}. Each dataset consists of 30 highly challenging, competition-level mathematics problems designed to test deep, multi-step analytical reasoning without intermediate execution feedback.
    \item \textbf{Code Generation:} We utilize \textbf{LiveCodeBench} \citep{jain2024livecodebench}, specifically a recent subset of 131 contest-style coding questions released between February and May 2025. This benchmark optionally provides rigorous execution feedback via public test cases.
\end{itemize}

We define two variants for our method based on the feedback: 
\begin{itemize}[leftmargin=*, noitemsep, topsep=0pt]
    \item \textbf{{\ttsd}-GenFB:} This uses generic feedback which includes previous answers, the full reasoning trace for the last attempt, a static error message (``Your previous attempt was unsuccessful.") and a re-attempt instruction. This is applicable for all benchmarks.
    \item \textbf{{\ttsd}-EnvFB:} The static error message in {\ttsd}-GenFB is replaced by rich environment feedback such as runtime errors or failing public test cases. This can be tested on LiveCodeBench as it provides execution based feedback, but not on maths benchmarks.
\end{itemize}


\paragraph{Implementation Details}
 We utilize the \textbf{Qwen3-8B} and \textbf{Qwen3-4B-Thinking-2507} models for our experiments. To compare inference-time scaling across methods, we report both the \emph{rollout budget} $k$ and the \emph{total number of generated tokens}. The maximum generation limit per attempt is capped at 16,384 tokens. Decoding across all methods uses a temperature $T=1$ and $\text{top}_p=1$. For {\ttsd}, we use a localized spike detection threshold of $\tau=0.06$ and maintain a maximum checkpoint depth of $D_c=16,384$. To standardize evaluation, our primary metric is \textbf{pass@$k$}, measuring whether at least one of the $k$ adaptively generated rollouts produces a correct solution.

\paragraph{Baselines}
To isolate the performance gains of {\ttsd}, we compare our method against three robust test-time scaling baselines:
\begin{enumerate}[leftmargin=*, noitemsep, topsep=0pt]
    \item \textbf{Independent Sampling:} A standard parallel generation (best-of-$k$) approach. The model independently generates $k$ full rollouts, and we evaluate whether any trajectory successfully solves the problem.
    \item \textbf{Multi-turn Refinement:} A sequential feedback-driven approach. The model is allowed a maximum of $k$ interactive turns to arrive at the correct answer, where each subsequent attempt is conditioned on the full sequence of previous attempts and their respective execution feedback.
    \item \textbf{Recursive Self-Aggregation (RSA):} A recent population-based test-time scaling algorithm that searches through a parameter space characterized by population size $N$, aggregation size $K$, and a distinct number of sequential steps $T$.
\end{enumerate}

\section{Experimental Results}
\label{sec:experimental_results}

We structure our results to directly address the research questions outlined in Section~\ref{sec:experiments}. We first analyze the compute-optimal scaling capabilities of {\ttsd} across domains (RQ1), followed by detailed ablations dissecting how structural context (RQ2) and filtering using a non-diagnostic feedback baseline (RQ3) drive its error localization precision.

\subsection{Compute-Optimal Scaling and Generalization (RQ1)}
\label{sec:results_scaling}

We evaluate {\ttsd} and our selected baselines across a varying rollout budget $k$ to analyze the scaling dynamics of pass@$k$ relative to overall token efficiency. We do not assume access to environment feedback and utilize the same generic error message as part of feedback in {\ttsd}-GenFB and Multi-Turn Refinement.

\FloatBarrier




\begin{figure*}[b]
    \centering
        \centering
        \includegraphics[width=\linewidth]{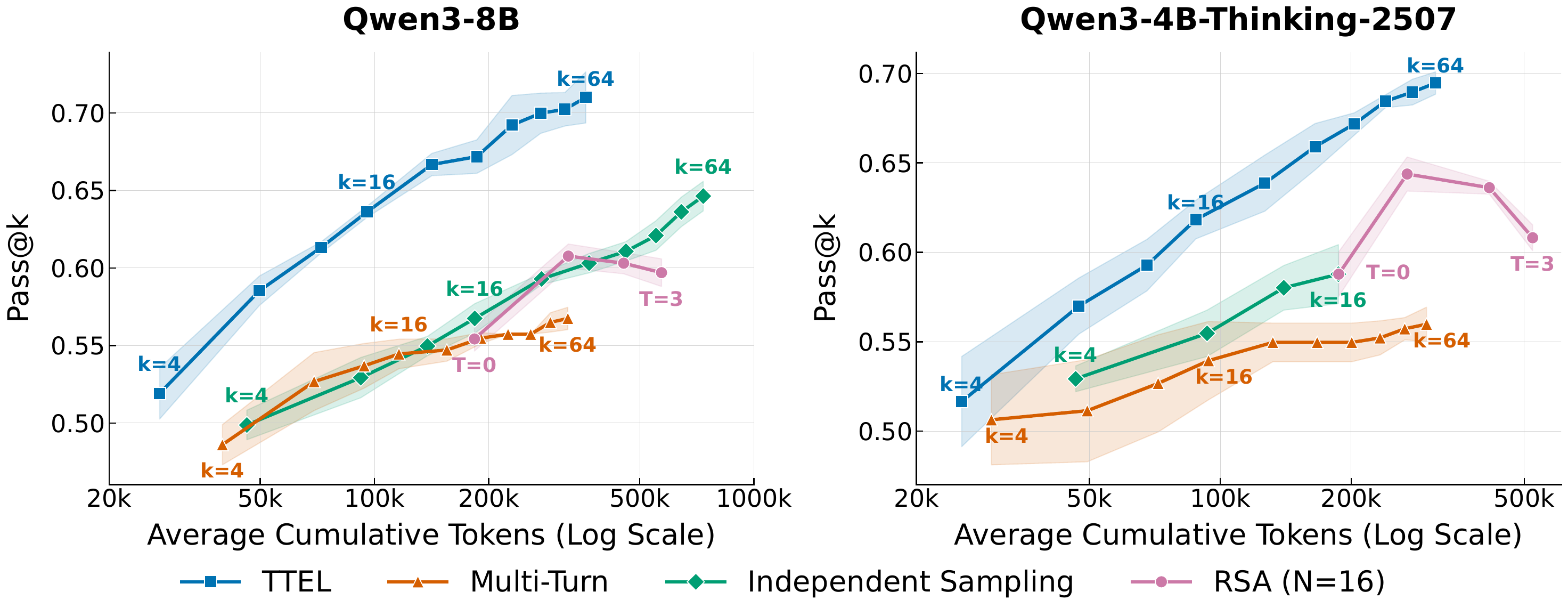}
        \caption{Pareto plot of pass@$k$ vs.\ average tokens generated per question on LiveCodeBench V6.}
        \label{fig:lcb_multi_pareto}
\end{figure*}

\begin{figure*}[t]
    \centering
        \centering
        \includegraphics[width=\linewidth]{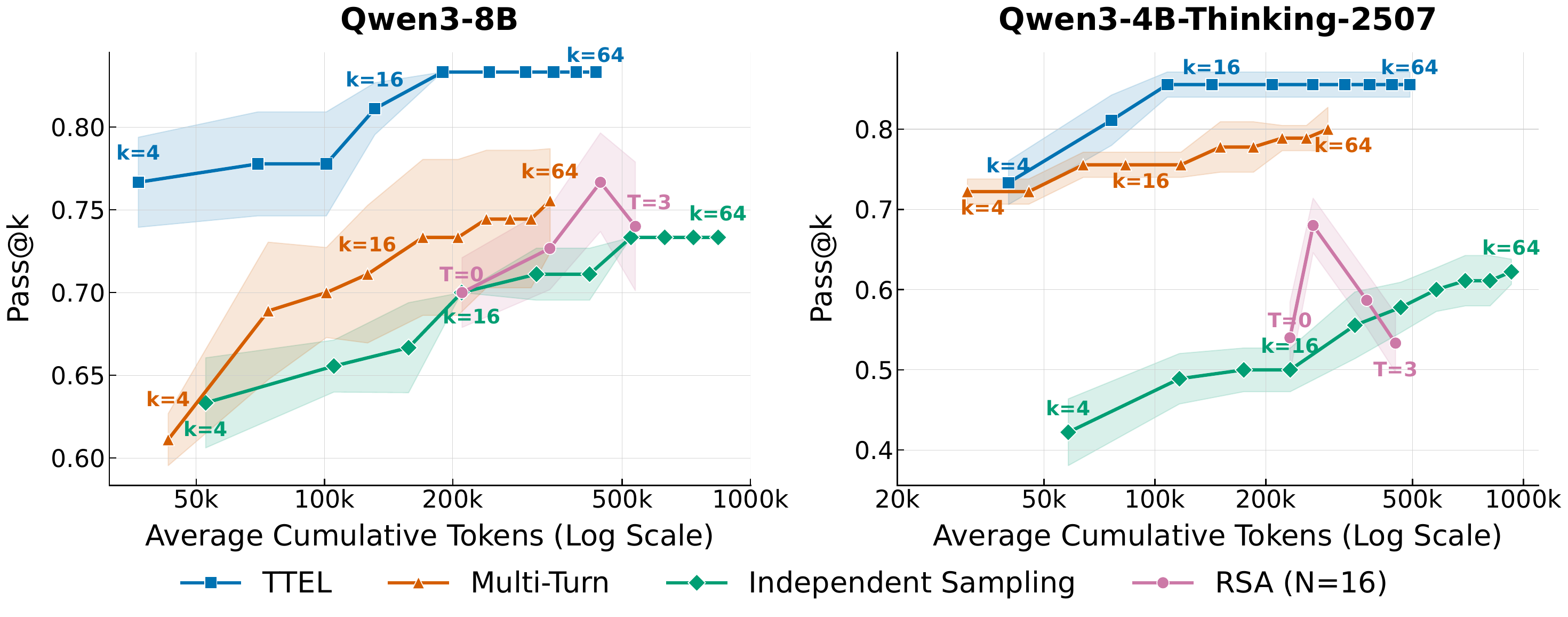}
        \caption{Pareto plot of pass@$k$ vs.\ average tokens generated per question on AIME-25.}
        \label{fig:aime_multi_pareto}
\end{figure*}

\begin{figure*}[t]
    \centering
        \centering
        \includegraphics[width=\linewidth]{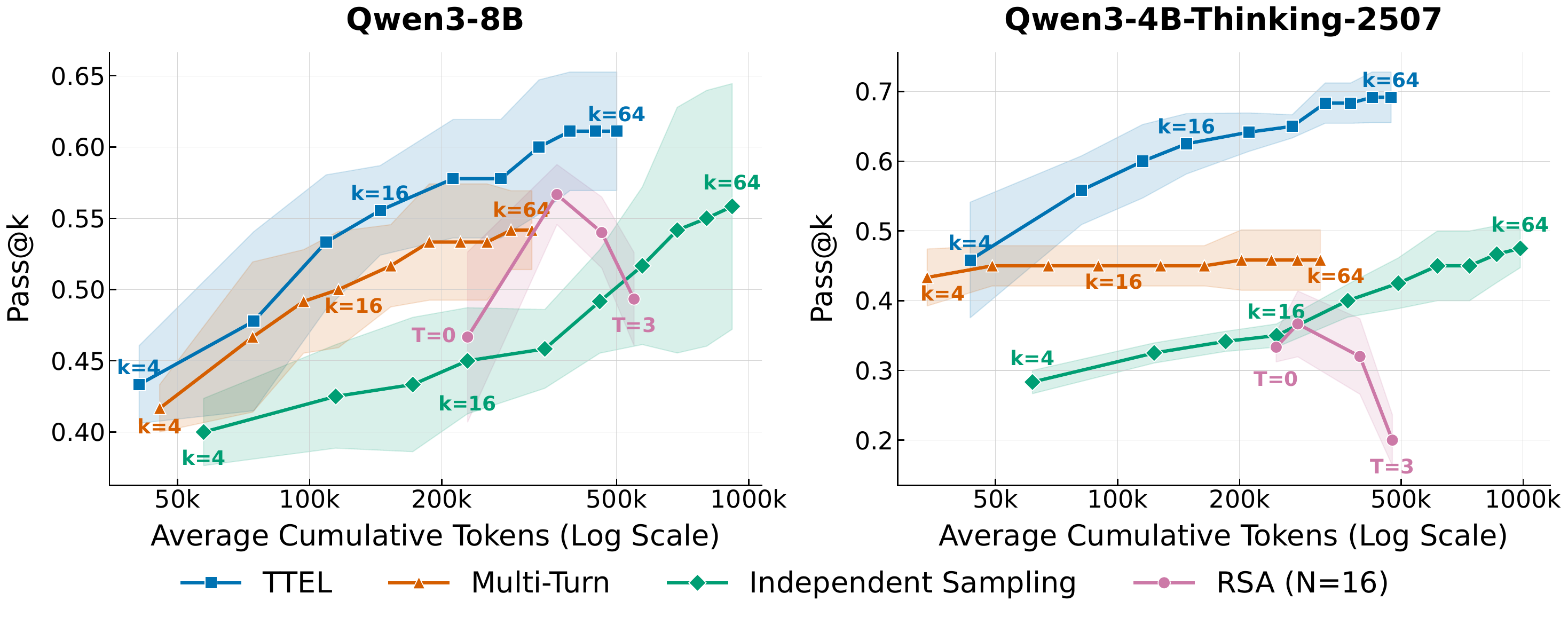}
        \caption{Pareto plot of pass@$k$ vs.\ average tokens generated per question on HMMT-25.}
        \label{fig:hmmt_multi_pareto}
\end{figure*}


\paragraph{Scaling on LiveCodeBench.}
Figure~\ref{fig:lcb_multi_pareto} illustrates the Pareto frontier of pass@$k$ relative to the average token cost per question on the LiveCodeBench V6 subset. We omit environment feedback (runtime errors, failing test cases) to demonstrate the effectiveness of {\ttsd} with just generic feedback. {\ttsd} establishes a strictly dominating Pareto frontier, achieving a superior pass@$k$ rate for any given generated-token budget. At a fixed maximum rollout budget of $k=64$, {\ttsd} achieves a pass@64 of 0.71. Notably, it outperforms both multi-turn refinement (which stagnates at 0.567) and independent sampling (0.646), while consuming less than half the average tokens of independent sampling. This compute efficiency validates that branching from retained prefixes prevents redundant regeneration. Note that the shaded region denotes the standard deviation for that point over 3 seeds.

\paragraph{Generalization to Math.}
We further evaluate scaling on the AIME-25 and HMMT-25 benchmarks, where the environment only provides terminal failure states without intermediate diagnostic feedback. Figure~\ref{fig:aime_multi_pareto} details the scaling trend on AIME-25 up to a maximum budget of $k=16$, where {\ttsd} reaches a pass@16 of 0.820, outperforming all baselines. This demonstrates that using generic feedback induces actionable token-level probability shifts, confirming that {\ttsd} operates as a domain-agnostic test-time scaling method. Note that the shaded region denotes the standard deviation for that point over 4 seeds.

\subsection{The Role of Context in Error Localization (RQ2)}
\label{sec:results_context}


\begin{figure}[htpb] 
    \centering
    \includegraphics[width=0.75\textwidth]{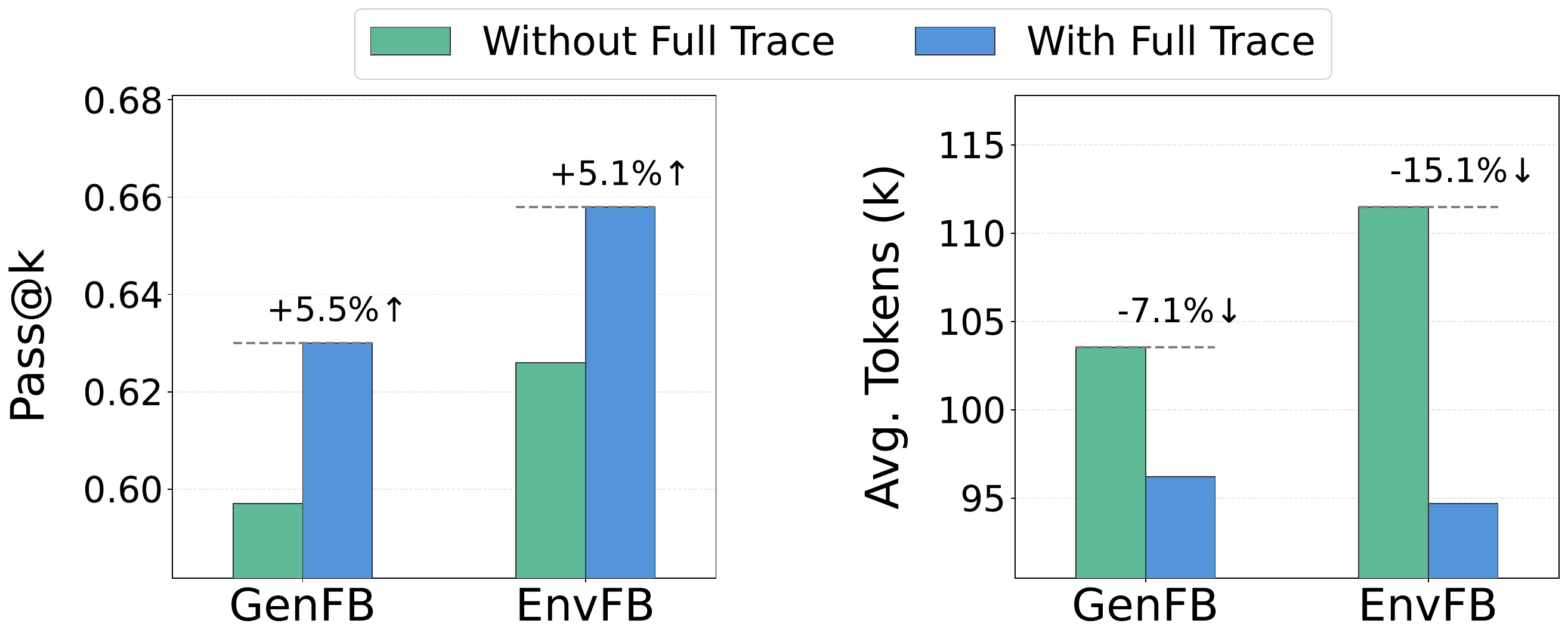}
    
    \caption{The impact of context ablation on LiveCodeBench V6 for Qwen3-8B ($k=16$). We evaluate both \ttsd-GenFB and \ttsd-EnvFB to measure the impact of informative environment feedback. To demonstrate the necessity of Full Thought Trace from the previous attempt, we include ablations that utilize only final answers from previous attempts.}

    \label{fig:ablation_context}
\end{figure}


Accurate token-level credit assignment depends on the  context provided to the teacher during spike detection. We hypothesize that retention of the original thought trace $y$, and the specificity of the external feedback $f$ can enable more precise error-localization. 

As illustrated in Figure~\ref{fig:ablation_context}, blinding the teacher to the \emph{Full Trace} generated by the student and conditioning execution failure solely on the extracted final answer causes a severe drop in pass@$k$ (from 0.630 to 0.597 in case of \emph{TTEL-GenFB}) alongside inflated token consumption. Accurate mistake localization algebraically requires the model to review the logic that actually produced the error, rendering the \emph{Full Trace} strictly necessary. Furthermore, advancing beyond generic feedback by injecting explicit execution traces aggressively grounds the probability distribution. This explicit external signal triggers far more precise log-probability divergences than implicit self-correction, lowering search cost and boosting pass@$k$ to 0.658 for \emph{TTEL-EnvFB}. Ultimately, token-efficient search relies on maximizing diagnostic context $f$ during re-scoring.

\vspace{-0.05in}
\subsection{Isolating Semantic Signals via Baseline Filtering (RQ3)}
\label{sec:results_filtering}

\vspace{-0.05in}
Finally, we validate the core architectural component of {\ttsd}: the unconditional (null) baseline filter, $\Delta_t(f_{\varnothing})$. While differences in student and teacher probabilities can identify divergence in reasoning due to genuine mistakes, the updated context for the teacher also induces a baseline probability shift. For example, the teacher may try to rephrase or abridge the reasoning trace of the student, leading to a large number of false-positive spikes that are unsuitable for diverging into novel branches. This is illustrated in  \cref{app:qualitative_example} in the Appendix.

\vspace{-0.05in}
\paragraph{Impact of Null-Baseline Spike Filtering.}
This ablation utilizes raw teacher probability spikes $\Delta_t(f)$ to trigger branching, entirely removing the secondary null-feedback baseline subtraction. As shown in Figure~\ref{fig:ablation_spike}, removing this filter dramatically changes the search dynamics. The average number of detected spikes in Turn 1 explodes from 19.3 to 486.0, representing massive false-positive mistake detections driven merely by context augmentation rather than genuine error correction. This noisy search degrades downstream pass@$k$ to 0.592. Thus, the baseline filter is strictly required to normalize the signal, filter context-induced probability shifts, and accurately localize actionable mathematical or logical errors.

\begin{figure}[H]
    \centering
    \includegraphics[width=\linewidth]{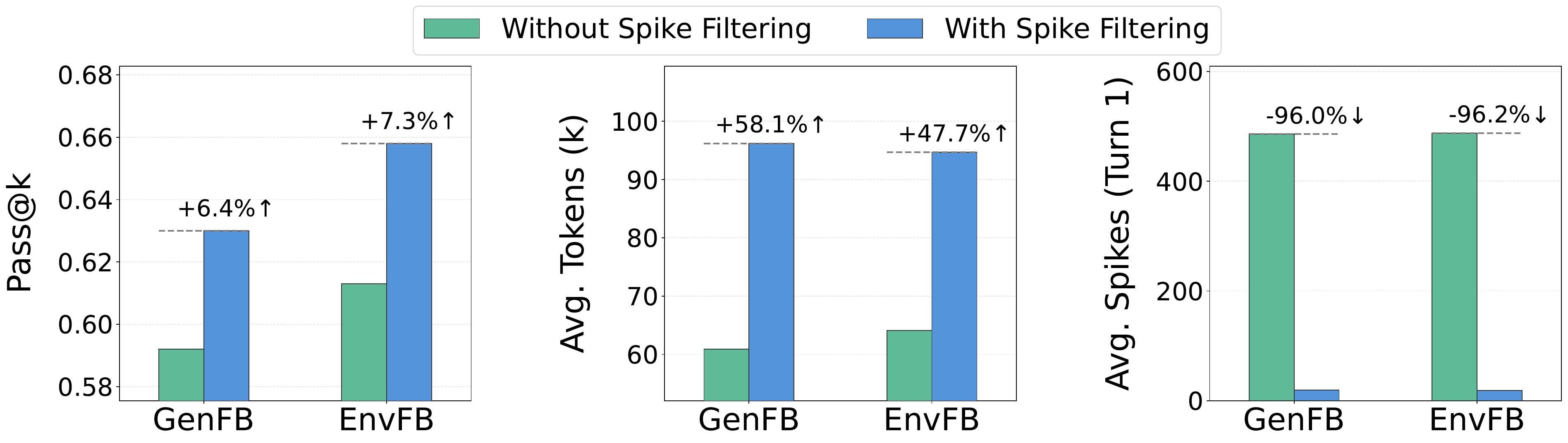}

    \caption{Null-baseline removal leads to massive false-positive spike detections and lowers accuracy on LiveCodeBench V6 for Qwen3-8B ($k=16$).\vspace{-0.1in}}
    \label{fig:ablation_spike}
\end{figure}

\vspace{-0.1in}
\paragraph{Conclusion: }
In this work, we introduced Test-Time Scaling via Error Localization ({\ttsd}), demonstrating that LLMs can perform precise token-level credit assignment at inference time without gradient updates. By coupling localized error detection directly with trajectory retention, {\ttsd} maximally reuses valid reasoning prefixes and mitigates computational redundancy.
Our results establish new compute-optimal Pareto frontiers, across math and coding tasks.




\bibliography{References}

@article{snell2024scaling,
  title={Scaling llm test-time compute optimally can be more effective than scaling model parameters},
  author={Snell, Charlie and Lee, Jaehoon and Xu, Kelvin and Kumar, Aviral},
  journal={arXiv preprint arXiv:2408.03314},
  year={2024}
}

@article{venkatraman2025recursive,
  title={Recursive self-aggregation unlocks deep thinking in large language models},
  author={Venkatraman, Siddarth and Jain, Vineet and Mittal, Sarthak and Shah, Vedant and Obando-Ceron, Johan and Bengio, Yoshua and Bartoldson, Brian R and Kailkhura, Bhavya and Lajoie, Guillaume and Berseth, Glen and others},
  journal={arXiv preprint arXiv:2509.26626},
  year={2025}
}

@article{yao2023tree,
  title={Tree of thoughts: Deliberate problem solving with large language models},
  author={Yao, Shunyu and Yu, Dian and Zhao, Jeffrey and Shafran, Izhak and Griffiths, Tom and Cao, Yuan and Narasimhan, Karthik},
  journal={Advances in neural information processing systems},
  volume={36},
  pages={11809--11822},
  year={2023}
}

@article{madaan2023self,
  title={Self-refine: Iterative refinement with self-feedback},
  author={Madaan, Aman and Tandon, Niket and Gupta, Prakhar and Hallinan, Skyler and Gao, Luyu and Wiegreffe, Sarah and Alon, Uri and Dziri, Nouha and Prabhumoye, Shrimai and Yang, Yiming and others},
  journal={Advances in neural information processing systems},
  volume={36},
  pages={46534--46594},
  year={2023}
}

@article{jain2024livecodebench,
  title={Livecodebench: Holistic and contamination free evaluation of large language models for code},
  author={Jain, Naman and Han, King and Gu, Alex and Li, Wen-Ding and Yan, Fanjia and Zhang, Tianjun and Wang, Sida and Solar-Lezama, Armando and Sen, Koushik and Stoica, Ion},
  journal={arXiv preprint arXiv:2403.07974},
  year={2024}
}

@article{hubotter2026reinforcement,
  title={Reinforcement Learning via Self-Distillation},
  author={H{\"u}botter, Jonas and L{\"u}beck, Frederike and Behric, Lejs and Baumann, Anton and Bagatella, Marco and Marta, Daniel and Hakimi, Ido and Shenfeld, Idan and Buening, Thomas Kleine and Guestrin, Carlos and others},
  journal={arXiv preprint arXiv:2601.20802},
  year={2026}
}

@inproceedings{wang2023h,
  title={H. Chi, Sharan Narang, Aakanksha Chowdhery, and Denny Zhou. Self-consistency improves chain of thought reasoning in language models},
  author={Wang, Xuezhi and Wei, Jason and Schuurmans, Dale and Le, Quoc V},
  booktitle={The Eleventh International Conference on Learning Representations},
  volume={1},
  pages={2},
  year={2023}
}

@article{liu2024rrm,
  title={Rrm: Robust reward model training mitigates reward hacking},
  author={Liu, Tianqi and Xiong, Wei and Ren, Jie and Chen, Lichang and Wu, Junru and Joshi, Rishabh and Gao, Yang and Shen, Jiaming and Qin, Zhen and Yu, Tianhe and others},
  journal={arXiv preprint arXiv:2409.13156},
  year={2024}
}

@article{srivastava2025robust,
  title={Robust Reward Modeling via Causal Rubrics},
  author={Srivastava, Pragya and Singh, Harman and Madhavan, Rahul and Patil, Gandharv and Addepalli, Sravanti and Suggala, Arun and Aravamudhan, Rengarajan and Sharma, Soumya and Laha, Anirban and Raghuveer, Aravindan and others},
  journal={arXiv preprint arXiv:2506.16507},
  year={2025}
}

@article{shao2024deepseekmath,
  title={Deepseekmath: Pushing the limits of mathematical reasoning in open language models},
  author={Shao, Zhihong and Wang, Peiyi and Zhu, Qihao and Xu, Runxin and Song, Junxiao and Bi, Xiao and Zhang, Haowei and Zhang, Mingchuan and Li, YK and Wu, Yang and others},
  journal={arXiv preprint arXiv:2402.03300},
  year={2024}
}

@inproceedings{gupta2025adaptive,
  title={Adaptive Control for Test-time Scaling},
  author={Gupta, Taneesh and Madhavan, Rahul and Tiwari, Rishabh and Zhang, Xuchao and Bansal, Chetan and Rajmohan, Saravan and Keutzer, Kurt},
  booktitle={The 5th Workshop on Mathematical Reasoning and AI at NeurIPS 2025},
  year={2025}
}

@article{chen2024not,
  title={Do not think that much for 2+ 3=? on the overthinking of o1-like llms},
  author={Chen, Xingyu and Xu, Jiahao and Liang, Tian and He, Zhiwei and Pang, Jianhui and Yu, Dian and Song, Linfeng and Liu, Qiuzhi and Zhou, Mengfei and Zhang, Zhuosheng and others},
  journal={arXiv preprint arXiv:2412.21187},
  year={2024}
}

@article{comanici2025gemini,
  title={Gemini 2.5: Pushing the frontier with advanced reasoning, multimodality, long context, and next generation agentic capabilities},
  author={Comanici, Gheorghe and Bieber, Eric and Schaekermann, Mike and Pasupat, Ice and Sachdeva, Noveen and Dhillon, Inderjit and Blistein, Marcel and Ram, Ori and Zhang, Dan and Rosen, Evan and others},
  journal={arXiv preprint arXiv:2507.06261},
  year={2025}
}

@article{luo2024improve,
  title={Improve mathematical reasoning in language models by automated process supervision},
  author={Luo, Liangchen and Liu, Yinxiao and Liu, Rosanne and Phatale, Samrat and Guo, Meiqi and Lara, Harsh and Li, Yunxuan and Shu, Lei and Zhu, Yun and Meng, Lei and others},
  journal={arXiv preprint arXiv:2406.06592},
  year={2024}
}

@inproceedings{gupta2025carmo,
  title={CARMO: Dynamic Criteria Generation for Context Aware Reward Modelling},
  author={Gupta, Taneesh and Shandilya, Shivam and Zhang, Xuchao and Madhavan, Rahul and Ghosh, Supriyo and Bansal, Chetan and Yao, Huaxiu and Rajmohan, Saravan},
  booktitle={Findings of the Association for Computational Linguistics: ACL 2025},
  pages={2202--2261},
  year={2025}
}

@article{zelikman2022star,
  title={Star: Bootstrapping reasoning with reasoning},
  author={Zelikman, Eric and Wu, Yuhuai and Mu, Jesse and Goodman, Noah},
  journal={Advances in Neural Information Processing Systems},
  volume={35},
  pages={15476--15488},
  year={2022}
}

@article{hosseini2024v,
  title={V-star: Training verifiers for self-taught reasoners},
  author={Hosseini, Arian and Yuan, Xingdi and Malkin, Nikolay and Courville, Aaron and Sordoni, Alessandro and Agarwal, Rishabh},
  journal={arXiv preprint arXiv:2402.06457},
  year={2024}
}

@inproceedings{wang2024math,
  title={Math-shepherd: Verify and reinforce llms step-by-step without human annotations},
  author={Wang, Peiyi and Li, Lei and Shao, Zhihong and Xu, Runxin and Dai, Damai and Li, Yifei and Chen, Deli and Wu, Yu and Sui, Zhifang},
  booktitle={Proceedings of the 62nd Annual Meeting of the Association for Computational Linguistics (Volume 1: Long Papers)},
  pages={9426--9439},
  year={2024}
}

@article{uesato2022solving,
  title={Solving math word problems with process-and outcome-based feedback},
  author={Uesato, Jonathan and Kushman, Nate and Kumar, Ramana and Song, Francis and Siegel, Noah and Wang, Lisa and Creswell, Antonia and Irving, Geoffrey and Higgins, Irina},
  journal={arXiv preprint arXiv:2211.14275},
  year={2022}
}

@article{chen2022program,
  title={Program of thoughts prompting: Disentangling computation from reasoning for numerical reasoning tasks},
  author={Chen, Wenhu and Ma, Xueguang and Wang, Xinyi and Cohen, William W},
  journal={arXiv preprint arXiv:2211.12588},
  year={2022}
}

@inproceedings{diao2024active,
  title={Active prompting with chain-of-thought for large language models},
  author={Diao, Shizhe and Wang, Pengcheng and Lin, Yong and Pan, Rui and Liu, Xiang and Zhang, Tong},
  booktitle={Proceedings of the 62nd Annual Meeting of the Association for Computational Linguistics (Volume 1: Long Papers)},
  pages={1330--1350},
  year={2024}
}

@inproceedings{gao2023pal,
  title={Pal: Program-aided language models},
  author={Gao, Luyu and Madaan, Aman and Zhou, Shuyan and Alon, Uri and Liu, Pengfei and Yang, Yiming and Callan, Jamie and Neubig, Graham},
  booktitle={International conference on machine learning},
  pages={10764--10799},
  year={2023},
  organization={PMLR}
}

@article{fu2022complexity,
  title={Complexity-based prompting for multi-step reasoning},
  author={Fu, Yao and Peng, Hao and Sabharwal, Ashish and Clark, Peter and Khot, Tushar},
  journal={arXiv preprint arXiv:2210.00720},
  year={2022}
}

@article{wu2024inference,
  title={Inference scaling laws: An empirical analysis of compute-optimal inference for problem-solving with language models},
  author={Wu, Yangzhen and Sun, Zhiqing and Li, Shanda and Welleck, Sean and Yang, Yiming},
  journal={arXiv preprint arXiv:2408.00724},
  year={2024}
}

@article{wei2022chain,
  title={Chain-of-thought prompting elicits reasoning in large language models},
  author={Wei, Jason and Wang, Xuezhi and Schuurmans, Dale and Bosma, Maarten and Xia, Fei and Chi, Ed and Le, Quoc V and Zhou, Denny and others},
  journal={Advances in neural information processing systems},
  volume={35},
  pages={24824--24837},
  year={2022}
}

@inproceedings{muennighoff2025s1,
  title={s1: Simple test-time scaling},
  author={Muennighoff, Niklas and Yang, Zitong and Shi, Weijia and Li, Xiang Lisa and Fei-Fei, Li and Hajishirzi, Hannaneh and Zettlemoyer, Luke and Liang, Percy and Cand{\`e}s, Emmanuel and Hashimoto, Tatsunori B},
  booktitle={Proceedings of the 2025 Conference on Empirical Methods in Natural Language Processing},
  pages={20286--20332},
  year={2025}
}

@inproceedings{lightman2023let,
  title={Let's verify step by step},
  author={Lightman, Hunter and Kosaraju, Vineet and Burda, Yuri and Edwards, Harrison and Baker, Bowen and Lee, Teddy and Leike, Jan and Schulman, John and Sutskever, Ilya and Cobbe, Karl},
  booktitle={The twelfth international conference on learning representations},
  year={2023}
}

@article{brown2024large,
  title={Large language monkeys: Scaling inference compute with repeated sampling},
  author={Brown, Bradley and Juravsky, Jordan and Ehrlich, Ryan and Clark, Ronald and Le, Quoc V and R{\'e}, Christopher and Mirhoseini, Azalia},
  journal={arXiv preprint arXiv:2407.21787},
  year={2024}
}

@article{zhao2026self,
  title={Self-Distilled Reasoner: On-Policy Self-Distillation for Large Language Models},
  author={Zhao, Siyan and Xie, Zhihui and Liu, Mengchen and Huang, Jing and Pang, Guan and Chen, Feiyu and Grover, Aditya},
  journal={arXiv preprint arXiv:2601.18734},
  year={2026}
}

@article{zhou2023language,
  title={Language agent tree search unifies reasoning acting and planning in language models},
  author={Zhou, Andy and Yan, Kai and Shlapentokh-Rothman, Michal and Wang, Haohan and Wang, Yu-Xiong},
  journal={arXiv preprint arXiv:2310.04406},
  year={2023}
}

@article{novikov2025alphaevolve,
  title={Alphaevolve: A coding agent for scientific and algorithmic discovery},
  author={Novikov, Alexander and V{\~u}, Ng{\^a}n and Eisenberger, Marvin and Dupont, Emilien and Huang, Po-Sen and Wagner, Adam Zsolt and Shirobokov, Sergey and Kozlovskii, Borislav and Ruiz, Francisco JR and Mehrabian, Abbas and others},
  journal={arXiv preprint arXiv:2506.13131},
  year={2025}
}

@article{shenfeld2026self,
  title={Self-Distillation Enables Continual Learning},
  author={Shenfeld, Idan and Damani, Mehul and H{\"u}botter, Jonas and Agrawal, Pulkit},
  journal={arXiv preprint arXiv:2601.19897},
  year={2026}
}

@article{song2026expanding,
  title={Expanding the Capabilities of Reinforcement Learning via Text Feedback},
  author={Song, Yuda and Chen, Lili and Tajwar, Fahim and Munos, Remi and Pathak, Deepak and Bagnell, J Andrew and Singh, Aarti and Zanette, Andrea},
  journal={arXiv preprint arXiv:2602.02482},
  year={2026}
}

@article{ye2026policy,
  title={On-Policy Context Distillation for Language Models},
  author={Ye, Tianzhu and Dong, Li and Wu, Xun and Huang, Shaohan and Wei, Furu},
  journal={arXiv preprint arXiv:2602.12275},
  year={2026}
}

@article{klissarov2026improving,
  title={Improving Interactive In-Context Learning from Natural Language Feedback},
  author={Klissarov, Martin and Cook, Jonathan and Antognini, Diego and Sun, Hao and Li, Jingling and Jaques, Natasha and Musat, Claudiu and Grefenstette, Edward},
  journal={arXiv preprint arXiv:2602.16066},
  year={2026}
}

@inproceedings{agarwal2024policy,
  title={On-policy distillation of language models: Learning from self-generated mistakes},
  author={Agarwal, Rishabh and Vieillard, Nino and Zhou, Yongchao and Stanczyk, Piotr and Garea, Sabela Ramos and Geist, Matthieu and Bachem, Olivier},
  booktitle={The twelfth international conference on learning representations},
  year={2024}
}

@article{kumar2024training,
  title={Training language models to self-correct via reinforcement learning},
  author={Kumar, Aviral and Zhuang, Vincent and Agarwal, Rishabh and Su, Yi and Co-Reyes, John D and Singh, Avi and Baumli, Kate and Iqbal, Shariq and Bishop, Colton and Roelofs, Rebecca and others},
  journal={arXiv preprint arXiv:2409.12917},
  year={2024}
}

@misc{deepseekai2025deepseekr1incentivizingreasoningcapability,
      title={DeepSeek-R1: Incentivizing Reasoning Capability in LLMs via Reinforcement Learning}, 
      author={DeepSeek-AI},
      year={2025},
      eprint={2501.12948},
      archivePrefix={arXiv},
      primaryClass={cs.CL},
      url={https://arxiv.org/abs/2501.12948}, 
}

@inproceedings{
balunovic2025matharena,
title={MathArena: Evaluating {LLM}s on Uncontaminated Math Competitions},
author={Mislav Balunovic and Jasper Dekoninck and Ivo Petrov and Nikola Jovanovi{\'c} and Martin Vechev},
booktitle={The Thirty-ninth Annual Conference on Neural Information Processing Systems Datasets and Benchmarks Track},
year={2025},
url={https://openreview.net/forum?id=y0zL9IZxZ7}
}
\bibliographystyle{colm2026_conference}

\newpage
\appendix

\vspace{1cm}
\hrule
\par\vspace{0.5cm}
{\Large\bfseries\centering \textsc 
{Supplementary Materials}
\par\vspace{0.5cm}}
\hrule
\vspace{0.5cm}


\section{Extended Related Work}
\label{app:related_work_extended}

\paragraph{Inference-time scaling via sampling, reranking, and verifier selection.}
The capability of large language models on complex reasoning and programming tasks scales predictably with inference-time compute \citep{snell2024scaling, brown2024large, wu2024inference, muennighoff2025s1, deepseekai2025deepseekr1incentivizingreasoningcapability}. The most straightforward and robust baseline in this family is independent parallel sampling, where massive arrays of complete candidates are generated and evaluated under specific scaling laws \citep{brown2024large}. Techniques such as pass@$k$, Best-of-$K$, majority voting, and verifier reranking \citep{lightman2023let} attempt to select the optimal solution from this generated pool. While these methods reliably convert increased test-time compute into higher accuracy, they are fundamentally limited by redundancy: they generate multiple full candidates entirely independently, failing to reuse trajectory-specific feedback or preserve valid partial progress from failed attempts.

\paragraph{Chain-of-thought prompting and structured reasoning elicitation.}
Chain-of-Thought (CoT) prompting \citep{wei2022chain} revealed that eliciting explicit intermediate reasoning steps dramatically improves language model performance on complex multi-step tasks. Self-Consistency \citep{wang2023h} extends this paradigm by sampling diverse CoT reasoning paths and aggregating via majority vote, effectively combining qualitative reasoning structure with parallel sampling. Further refinements include complexity-based prompting, which preferentially selects more elaborate reasoning chains \citep{fu2022complexity}; Program-of-Thought prompting, which externalizes computational steps to code interpreters \citep{chen2022program, gao2023pal}; and active prompting, which strategically selects the most informative exemplars to elicit stronger chains \citep{diao2024active}. These methods collectively establish that the \emph{structure} and \emph{quality} of reasoning is often the critical determinant of test-time performance. {\ttsd} deepens this insight: rather than eliciting better reasoning through prompt engineering, it uses a feedback-conditioned token-level signal to diagnose precisely \emph{where} reasoning quality degrades and surgically branches from that exact locus.

\paragraph{Sequential refinement and interactive feedback at inference time.}
To address the inefficiency of independent sampling, sequential refinement methods prepend execution feedback to prompt the model to iteratively correct its own errors \citep{madaan2023self}. Recent work investigates interactive in-context learning from natural language feedback as a distinctly trainable capability \citep{klissarov2026improving}. While these approaches successfully leverage environmental critiques, the setting remains fundamentally multi-turn in-context learning. Treating trajectory-specific feedback as a global instruction for a sequential retry forces the model to regenerate the reasoning chain from scratch, frequently resulting in context bloat, hallucination, or the repetition of earlier logical errors \citep{kumar2024training}. {\ttsd} mitigates this by using feedback not for an unanchored sequential retry, but for prefix-anchored branching.

\paragraph{Tree search and structured exploration over reasoning traces.}
To overcome the limitations of sequential generation, backtracking algorithms such as Tree of Thoughts \citep{yao2023tree} and Language Agent Tree Search \citep{zhou2023language} explicitly structure inference as a search space over partial reasoning trajectories. Population-based and evolutionary methods, including Recursive Self-Aggregation \citep{venkatraman2025recursive} and AlphaEvolve \citep{novikov2025alphaevolve}, scale this by iteratively mutating and recombining successful reasoning paths to discover complex solutions. These methods establish language model inference as a formal search procedure. Concurrently, inference-time control algorithms explicitly optimize computational budgets utilizing explicit constraints, either by deploying verifiable testing mechanisms and adaptive early-stopping criteria to dynamically bound test-time horizons \citep{gupta2025adaptive}, or by applying simpler sequential budget-forcing sequences \citep{muennighoff2025s1}. However, a primary computational bottleneck across these tree-search and explicit control frameworks remains their reliance on manual, coarse-grained chunking, heuristic bounds, or whole-trajectory mutation: they divide generation into heuristically defined ``steps'' and evaluate states using external value functions. {\ttsd} is fundamentally a tree-search method, but one whose branching rule emerges organically from feedback-conditioned token disagreement rather than hand-designed programmatic search heuristics or blind mutation.

\paragraph{Verifier-guided reasoning, execution feedback, and program repair.}
In domains such as competitive programming and mathematical reasoning, execution environments provide rich diagnostic feedback, including runtime errors and public test case outputs \citep{jain2024livecodebench}. Previous literature in automated program repair and verifier-guided generation uses this execution feedback primarily to score, rerank, or prompt complete revisions of generated candidates. A key contrast of our work is that {\ttsd} utilizes verifier feedback not merely to evaluate candidates, but to explicitly decide \emph{where} to truncate the trajectory and branch a new attempt.

\paragraph{Process reward models and step-level credit assignment.}
Outcome reward models assign a single scalar to complete solutions, limiting their capacity for localizing errors within multi-step reasoning traces. Process reward models (PRMs) address this gap by providing step-level supervision, either through human annotation \citep{lightman2023let, uesato2022solving} or automated Monte Carlo estimation procedures such as Math-Shepherd \citep{wang2024math} and OmegaPRM \citep{luo2024improve}. Causal reward modelling techniques such as CROME \citep{srivastava2025robust}, RRM \citep{liu2024rrm} and CARMO \citep{gupta2025carmo} also attempt to address credit assignment through causal perturbations. The key limitation is that PRMs require training, introducing additional annotation costs and potential distribution mismatch when applied to new domains or model families. {\ttsd} achieves analogous credit assignment, but derives its signal purely from the probabilistic disagreement between the base and feedback-conditioned policies, requiring no externally trained reward model.

\paragraph{On-policy distillation and learning from self-generated trajectories.}
Because environment verifiers typically provide only terminal outcomes, minimizing the distribution mismatch between training and inference requires on-policy algorithms. Generalized Knowledge Distillation (GKD) trains the student model on its own generated output sequences while leveraging feedback from the teacher, preventing off-policy divergence \citep{agarwal2024policy}. Similarly, On-Policy Self-Distillation (OPSD) utilizes a single model acting as both teacher (conditioned on privileged verified reasoning traces) and student, minimizing the divergence over the student's own rollouts \citep{zhao2026self}. These highly relevant methods learn directly from the model's own trajectories, but they internalize the signal into model parameters through offline training, rather than dynamically guiding inference-time branching on the current problem instance.

\paragraph{Self-distillation and RL from rich textual feedback.}
Recent advancements have formalized reinforcement learning protocols capable of exploiting rich, textual environment feedback to overcome the scalar credit assignment bottleneck. Self-Distillation Policy Optimization (SDPO) converts tokenized feedback into a dense learning signal by treating the feedback-conditioned model as a self-teacher, distilling its next-token predictions back into the policy \citep{hubotter2026reinforcement}. Concurrently, RLTF explores self-distillation and feedback-modeling objectives to internalize multi-turn textual critique for improved single-turn test-time performance \citep{song2026expanding}. Conceptually, these works are the closest to our approach, as they utilize textual critique for dense token-level credit assignment. However, they primarily improve the policy through offline training. {\ttsd} demonstrates that similar token-level divergence signals can be repurposed to guide search organically \emph{without} gradient updates.

\paragraph{Training for extended reasoning via reinforcement learning.}
A parallel line of work focuses on training models to produce longer, more deliberate reasoning traces. STaR \citep{zelikman2022star} and V-STaR \citep{hosseini2024v} bootstrap reasoning ability by iteratively fine-tuning on self-generated correct rationales. At larger scale, Gemini 2.5 \citep{comanici2025gemini} employs reinforcement learning to produce extended internal chains of thought, while DeepSeek-R1 \citep{deepseekai2025deepseekr1incentivizingreasoningcapability} demonstrates that RL with verifier rewards can elicit sophisticated emergent behaviors including self-verification and backtracking. GRPO \citep{shao2024deepseekmath} and related group-relative policy optimization methods reduce the need for explicit critic models in this setting. These training-time investments produce reasoning models whose outputs are naturally amenable to test-time scaling interventions. {\ttsd} is complementary to this paradigm: it operates purely at inference time on any model capable of generating reasoning traces, requiring no additional training or parameter updates, and can be composed with RL-trained reasoning models to further improve their test-time search efficiency.

\paragraph{Context distillation and internalizing auxiliary information into weights.}
Context distillation methods aim to convert useful but cumbersome context, such as optimized system prompts, retrieval augmentations, or historical traces, directly into model parameters. For instance, On-Policy Context Distillation (OPCD) bridges on-policy learning with context distillation by training student models on generated trajectories to match a context-conditioned teacher \citep{ye2026policy}. {\ttsd} is related to this paradigm in spirit, as it also leverages the auxiliary information provided by prior attempts and environment feedback. However, rather than absorbing that information offline to permanently update weights, {\ttsd} uses the auxiliary feedback online to evaluate and anchor the search at a securely retained prefix.

\paragraph{Self-distillation beyond post-training: demonstrations and continual learning.}
The utility of on-policy self-distillation extends beyond reasoning verification. Self-Distillation Fine-Tuning (SDFT) establishes it as a powerful recipe for continual learning from demonstrations, leveraging a demonstration-conditioned model as an on-policy teacher to acquire new skills while avoiding catastrophic forgetting \citep{shenfeld2026self}. While less directly focused on verifier-guided reasoning, SDFT highlights that context-conditioned self-teaching signals are part of a broader, highly effective algorithmic trend. {\ttsd} positions itself within this trend, uniquely distinct by deploying the self-teaching signal purely at inference time.

\paragraph{Overthinking, reasoning efficiency, and adaptive compute allocation.}
As reasoning models generate increasingly long chains of thought, the problem of \emph{overthinking} expending substantial compute on problems solvable with far less deliberation has emerged as a practical concern \citep{chen2024not, gupta2025adaptive}. Adaptive compute allocation methods seek to match reasoning effort to problem difficulty, for instance through early-exit mechanisms, confidence-based stopping criteria, or learned allocation policies. {\ttsd}'s spike-based branching mechanism naturally addresses this inefficiency: by localizing errors to specific token positions via feedback-conditioned disagreement, it preserves demonstrably correct reasoning prefixes and avoids redundant regeneration of already-valid computation. This prefix-anchored branching strategy ensures that test-time compute is concentrated precisely where reasoning fails, rather than distributed uniformly or wastefully across the full trajectory.

\paragraph{Summary of Positioning.}
In summary, {\ttsd} lies at the intersection of verifier-guided inference-time scaling, token-level self-distillation-style credit assignment, and structured tree search. Its distinguishing conceptual feature is that it couples feedback-conditioned credit assignment with retained-prefix branching on the exact failed trajectory. By doing so, it transcends the limitations of methods that utilize feedback exclusively for outcome reranking, unanchored multi-turn retries, or static offline parameter updates.

\section{Experimental Configuration and Prompt Templates}
\label{app:experimental_config}

\paragraph{Generation Settings}
We experiment with the \textbf{Qwen3-8B} and \textbf{Qwen3-4B-Thinking-2507} models. The maximum generation limit per attempt is capped at 16,384 tokens. Decoding across all methods uses a temperature $T=1$ and $\text{top}_p=1$. 

\paragraph{Test-Time Scaling via Error Localization}
For {\ttsd}, we use a localized spike detection threshold of $\tau=0.06$ and $\tau_B=0.06$ for the null-baseline. We maintain a maximum checkpoint depth of $D_c=16,384$.

At the first turn or after a restart, the user prompt is the problem statement.
For further turns, the user prompt also includes a history of prior answers, respective feedback and the reasoning trace for the final attempt. 

\begin{tcolorbox}[colback=gray!5, colframe=gray!50, fontupper=\small, title=Student Prompt (With History)]
\texttt{\{question\}}\\[4pt]
\texttt{\#\#\# Attempt:}\\
\texttt{The following is a previous attempt to solve the question:}\\[2pt]
\texttt{\{answer$_{1}$\}} \hfill\textit{\small(answer-only for earlier turns)}\\[4pt]
\texttt{\#\#\# Feedback:}\\
\texttt{The following is feedback from your unsuccessful earlier attempt:}\\[2pt]
\texttt{\{feedback$_{1}$\}}\\[2pt]
\texttt{Solve the original question. Use the feedback to make changes.}\\[6pt]
$\vdots$ \hfill\textit{\small(repeat for turns $2 \ldots t{-}1$)}\\[6pt]
\texttt{\#\#\# Attempt:}\\
\texttt{The following is a previous attempt to solve the question:}\\[2pt]
\texttt{\{full\_trace$_{t-1}$\}} \hfill\textit{\small(full reasoning+answer for last turn only)}\\[4pt]
\texttt{\#\#\# Feedback:}\\
\texttt{The following is feedback from your unsuccessful earlier attempt:}\\[2pt]
\texttt{\{feedback$_{t-1}$\}}\\[2pt]
\texttt{Solve the original question. Use the feedback to make changes.}
\end{tcolorbox}

Teacher Prompt uses the same format but it includes information from the latest rollout generated by the student. Teacher Prompt at Turn~t becomes Student Prompt at Turn~t+1.

Spike-Baseline uses the Teacher Prompt but replaces the instruction 
``Solve the original question. Use the feedback to make changes." with
``Repeat the previous attempt word by word, but skip portions of redundant thinking."

\paragraph{Independent Sampling}
$k$ independent rollouts are sampled from a single prompt. The prompt is simply the problem statement.

\begin{tcolorbox}[colback=gray!5, colframe=gray!50, fontupper=\small, title=Independent Sampling Prompt]
\texttt{\{question\}}
\end{tcolorbox}

\paragraph{Multi-Turn Refinement}
This method carries out sequential refinement of answers by conditioning on a history of prior answers and feedback.

\begin{tcolorbox}[colback=gray!5, colframe=gray!50, fontupper=\small, title=Multi-Turn prompt (turn $t > 0$)]
\texttt{\{question\}}\\[4pt]
\texttt{\#\#\# Attempt:}\\
\texttt{The following is a previous attempt to solve the question:}\\[2pt]
\texttt{\{non\_think\_response$_{i}$\}} \hfill\textit{\small(answer-only for all turns)}\\[4pt]
\texttt{\#\#\# Feedback:}\\
\texttt{The following is feedback from your unsuccessful earlier attempt:}\\[2pt]
\texttt{\{feedback$_{i}$\}}\\[2pt]
\texttt{Solve the original question. Use the feedback to make changes.}\\[6pt]
$\vdots$ \hfill\textit{\small(repeat for all previous turns)}
\end{tcolorbox}

\paragraph{Recursive Self-Aggregation (RSA)}
We run RSA with a population of $N=16$ over $T=4$ turns and aggregation size $K=4$. The pass@$k$ values shown in plots are calculated over the $k=16$ rollouts present in the population after a turn T.
At Turn~0, $N$ independent rollouts are generated from the question prompt.
At Turns~1 to 3, $K$ candidate answers are sampled from the current population and presented to the model:

\begin{tcolorbox}[colback=gray!5, colframe=gray!50, fontupper=\small, title=RSA aggregation prompt ($K > 1$)]
\texttt{You are given a \{problem\_kind\} and several candidate solutions. Some candidates may be incorrect or contain errors. Aggregate the useful ideas and produce a single, high-quality solution. Reason carefully; if candidates disagree, choose the correct path. If all are incorrect, then attempt a different strategy. End with the final result in \{format\_hint\}.}\\[4pt]
\texttt{Problem:}\\
\texttt{\{question\}}\\[4pt]
\texttt{Candidate solutions (may contain mistakes):}\\[2pt]
\texttt{---- Solution 1 ----}\\
\texttt{\{answer$_1$\}}\\[2pt]
$\vdots$\\[2pt]
\texttt{---- Solution $K$ ----}\\
\texttt{\{answer$_K$\}}\\[4pt]
\texttt{Now write a single improved solution. Provide clear reasoning and end with the final answer in \{format\_hint\}.}
\end{tcolorbox}

The \texttt{format\_hint} adapts by domain. For example, \verb|\boxed{}| for math.




\section{Detailed Results for Scaling Experiments}
\label{app:detailed_results_scaling}

\noindent
This section provides detailed results and error analysis for scaling experiments with a varying rollout budget $k$. Experiments are conducted over 2 models, 3 datasets and repeated over different seeds, denoted by $s$.

AIME-25 and HMMT-25 have 30 questions each as opposed to LCBv6 with 131 questions. We evaluate the latter with 3 seeds and the former with 4 seeds to account for higher variance.

The reported token count is the average number of tokens generated per question. For RSA, pass@$k$ evaluation is over for a constant $k=16$, which is the population of $N=16$ rollouts generated after turn $i$ (denoted by $t_i$). All other methods scale up to a maximum budget of $k=64$.

\begin{table}[bh]
\centering
\small
\setlength{\tabcolsep}{4pt}
\begin{tabular}{l c crrr crrr}
\toprule
 &  & \multicolumn{4}{c}{\textbf{pass@$k$ (\%)}} & \multicolumn{4}{c}{\textbf{Generated Tokens ($\times 10^3$)}} \\
\cmidrule(lr){3-6}\cmidrule(lr){7-10}
Method & $k$ & Mean $\pm$ SD & $s$=0 & $s$=1 & $s$=2 & Mean $\pm$ SD & $s$=0 & $s$=1 & $s$=2 \\
\midrule
\multirow{10}{*}{TTEL} & 4 & $51.9 \pm 1.6$ & 54.2 & 51.1 & 50.4 & $27.2 \pm 0.5$ & 27.3 & 26.6 & 27.7 \\
 & 8 & $58.5 \pm 1.0$ & 58.8 & 57.3 & 59.5 & $49.7 \pm 0.4$ & 50.2 & 49.3 & 49.7 \\
 & 12 & $61.3 \pm 0.4$ & 61.8 & 61.1 & 61.1 & $72.4 \pm 0.2$ & 72.6 & 72.4 & 72.0 \\
 & 16 & $63.6 \pm 0.4$ & 63.4 & 64.1 & 63.4 & $95.6 \pm 0.2$ & 95.4 & 95.8 & 95.6 \\
 & 24 & $66.7 \pm 0.7$ & 67.2 & 65.6 & 67.2 & $141.7 \pm 0.9$ & 140.5 & 142.7 & 141.8 \\
 & 32 & $67.2 \pm 1.1$ & 67.9 & 65.6 & 67.9 & $186.0 \pm 0.6$ & 186.4 & 185.1 & 186.4 \\
 & 40 & $69.2 \pm 1.9$ & 71.8 & 67.2 & 68.7 & $230.3 \pm 0.3$ & 230.3 & 230.0 & 230.6 \\
 & 48 & $70.0 \pm 1.3$ & 71.8 & 69.5 & 68.7 & $274.1 \pm 0.8$ & 274.4 & 273.1 & 274.9 \\
 & 56 & $70.2 \pm 1.1$ & 71.8 & 69.5 & 69.5 & $317.6 \pm 0.4$ & 317.8 & 317.0 & 317.9 \\
 & 64 & $71.0 \pm 1.6$ & 73.3 & 70.2 & 69.5 & $360.4 \pm 0.7$ & 361.3 & 360.4 & 359.6 \\
\midrule
\multirow{10}{*}{\shortstack[l]{Independent\\Sampling}} & 4 & $49.9 \pm 1.0$ & 49.6 & 48.9 & 51.1 & $46.1 \pm 0.1$ & 46.1 & 46.2 & 46.0 \\
 & 8 & $52.9 \pm 1.3$ & 54.2 & 51.1 & 53.4 & $92.0 \pm 0.3$ & 92.0 & 92.4 & 91.8 \\
 & 12 & $55.0 \pm 0.6$ & 55.0 & 54.2 & 55.7 & $137.9 \pm 0.3$ & 137.9 & 138.2 & 137.5 \\
 & 16 & $56.7 \pm 1.0$ & 56.5 & 55.7 & 58.0 & $183.7 \pm 0.4$ & 184.0 & 183.9 & 183.2 \\
 & 24 & $59.3 \pm 0.4$ & 58.8 & 59.5 & 59.5 & $275.6 \pm 0.4$ & 275.8 & 276.0 & 275.0 \\
 & 32 & $60.3 \pm 0.6$ & 59.5 & 60.3 & 61.1 & $367.7 \pm 0.6$ & 368.4 & 367.8 & 366.8 \\
 & 40 & $61.1 \pm 0.6$ & 60.3 & 61.1 & 61.8 & $459.7 \pm 0.8$ & 460.8 & 459.4 & 458.9 \\
 & 48 & $62.1 \pm 1.0$ & 61.8 & 61.1 & 63.4 & $551.6 \pm 1.0$ & 553.0 & 551.1 & 550.8 \\
 & 56 & $63.6 \pm 1.0$ & 64.9 & 62.6 & 63.4 & $643.5 \pm 0.8$ & 644.6 & 643.1 & 642.8 \\
 & 64 & $64.6 \pm 1.0$ & 64.9 & 63.4 & 65.6 & $735.0 \pm 1.1$ & 736.5 & 734.6 & 734.0 \\
\midrule
\multirow{10}{*}{\shortstack[l]{Multi-Turn\\Refinement}} & 4 & $48.6 \pm 1.3$ & 47.3 & 50.4 & 48.1 & $39.7 \pm 0.1$ & 39.8 & 39.6 & 39.8 \\
 & 8 & $52.7 \pm 1.9$ & 52.7 & 55.0 & 50.4 & $69.3 \pm 0.6$ & 69.5 & 68.6 & 69.9 \\
 & 12 & $53.7 \pm 1.4$ & 52.7 & 55.7 & 52.7 & $93.9 \pm 1.0$ & 94.2 & 92.7 & 95.0 \\
 & 16 & $54.5 \pm 1.0$ & 53.4 & 55.7 & 54.2 & $116.0 \pm 1.1$ & 116.2 & 114.6 & 117.2 \\
 & 24 & $54.7 \pm 0.7$ & 54.2 & 55.7 & 54.2 & $155.1 \pm 1.1$ & 155.4 & 153.6 & 156.3 \\
 & 32 & $55.5 \pm 0.4$ & 55.0 & 55.7 & 55.7 & $190.6 \pm 1.8$ & 190.8 & 188.3 & 192.6 \\
 & 40 & $55.7 \pm 0.0$ & 55.7 & 55.7 & 55.7 & $224.9 \pm 2.6$ & 225.6 & 221.4 & 227.6 \\
 & 48 & $55.7 \pm 0.0$ & 55.7 & 55.7 & 55.7 & $257.9 \pm 2.7$ & 259.1 & 254.2 & 260.6 \\
 & 56 & $56.5 \pm 0.6$ & 56.5 & 55.7 & 57.3 & $290.7 \pm 2.9$ & 291.9 & 286.7 & 293.4 \\
 & 64 & $56.7 \pm 0.7$ & 57.3 & 55.7 & 57.3 & $323.2 \pm 2.7$ & 324.8 & 319.4 & 325.4 \\
\midrule
\multirow{4}{*}{RSA} & 16 ($t_{0}$) & $55.0 \pm 0.6$ & 54.2 & 55.7 & 55.0 & $183.4 \pm 0.3$ & 183.3 & 183.1 & 183.8 \\
 & 16 ($t_{1}$) & $60.3 \pm 0.6$ & 59.5 & 61.1 & 60.3 & $324.4 \pm 0.3$ & 324.2 & 324.2 & 324.8 \\
 & 16 ($t_{2}$) & $60.3 \pm 0.6$ & 60.3 & 59.5 & 61.1 & $453.7 \pm 1.3$ & 452.1 & 453.9 & 455.2 \\
 & 16 ($t_{3}$) & $60.1 \pm 0.4$ & 60.3 & 59.5 & 60.3 & $570.7 \pm 1.8$ & 568.7 & 570.3 & 573.1 \\
\bottomrule
\end{tabular}
\vspace{1mm}
\caption{Inference scaling and error analysis for Qwen3-8B on LiveCodeBench (LCB). We evaluate pass@$k$ (\%) and average generation cost per question over 3 different seeds.}
\label{tab:detailed_performance_S2_Qwen3-8B_k64_lcb}
\end{table}

\begin{table}[hp]
\centering
\small
\setlength{\tabcolsep}{4pt}
\begin{tabular}{l c crrrr crrrr}
\toprule
 &  & \multicolumn{5}{c}{\textbf{pass@$k$ (\%)}} & \multicolumn{5}{c}{\textbf{Generated Tokens ($\times 10^3$)}} \\
\cmidrule(lr){3-7}\cmidrule(lr){8-12}
Method & $k$ & Mean $\pm$ SD & $s$=0 & $s$=1 & $s$=2 & $s$=3 & Mean $\pm$ SD & $s$=0 & $s$=1 & $s$=2 & $s$=3 \\
\midrule
\multirow{10}{*}{TTEL} & 4 & $75.0 \pm 3.7$ & 76.7 & 80.0 & 73.3 & 70.0 & $36.6 \pm 0.7$ & 35.5 & 37.6 & 36.6 & 36.6 \\
 & 8 & $78.3 \pm 2.9$ & 80.0 & 80.0 & 73.3 & 80.0 & $70.0 \pm 0.5$ & 69.9 & 70.4 & 69.1 & 70.5 \\
 & 12 & $79.2 \pm 3.6$ & 80.0 & 80.0 & 73.3 & 83.3 & $101.5 \pm 1.1$ & 99.8 & 101.6 & 101.6 & 102.9 \\
 & 16 & $81.7 \pm 1.7$ & 83.3 & 80.0 & 80.0 & 83.3 & $132.4 \pm 2.4$ & 129.7 & 130.5 & 133.8 & 135.6 \\
 & 24 & $83.3 \pm 0.0$ & 83.3 & 83.3 & 83.3 & 83.3 & $192.4 \pm 5.1$ & 189.5 & 189.1 & 189.8 & 201.3 \\
 & 32 & $83.3 \pm 0.0$ & 83.3 & 83.3 & 83.3 & 83.3 & $248.2 \pm 7.8$ & 246.1 & 243.2 & 242.1 & 261.5 \\
 & 40 & $83.3 \pm 0.0$ & 83.3 & 83.3 & 83.3 & 83.3 & $301.3 \pm 10.6$ & 306.3 & 295.2 & 288.0 & 315.7 \\
 & 48 & $83.3 \pm 0.0$ & 83.3 & 83.3 & 83.3 & 83.3 & $351.2 \pm 15.2$ & 361.3 & 343.8 & 330.2 & 369.4 \\
 & 56 & $83.3 \pm 0.0$ & 83.3 & 83.3 & 83.3 & 83.3 & $396.3 \pm 16.2$ & 406.9 & 392.2 & 372.0 & 414.3 \\
 & 64 & $83.3 \pm 0.0$ & 83.3 & 83.3 & 83.3 & 83.3 & $439.5 \pm 18.0$ & 452.5 & 437.0 & 411.0 & 457.3 \\
\midrule
\multirow{10}{*}{\shortstack[l]{Independent\\Sampling}} & 4 & $64.2 \pm 2.8$ & 60.0 & 66.7 & 63.3 & 66.7 & $52.7 \pm 0.6$ & 52.5 & 51.9 & 53.5 & 52.8 \\
 & 8 & $65.8 \pm 1.4$ & 66.7 & 66.7 & 63.3 & 66.7 & $104.9 \pm 0.8$ & 105.2 & 104.6 & 106.0 & 104.0 \\
 & 12 & $67.5 \pm 2.8$ & 70.0 & 66.7 & 63.3 & 70.0 & $157.3 \pm 0.8$ & 157.2 & 156.9 & 158.7 & 156.6 \\
 & 16 & $70.0 \pm 0.0$ & 70.0 & 70.0 & 70.0 & 70.0 & $209.7 \pm 0.8$ & 209.8 & 209.2 & 211.0 & 209.0 \\
 & 24 & $70.8 \pm 1.4$ & 70.0 & 70.0 & 73.3 & 70.0 & $314.3 \pm 1.0$ & 314.4 & 313.6 & 315.9 & 313.4 \\
 & 32 & $70.8 \pm 1.4$ & 70.0 & 70.0 & 73.3 & 70.0 & $418.8 \pm 1.1$ & 419.0 & 418.0 & 420.5 & 417.8 \\
 & 40 & $72.5 \pm 1.4$ & 73.3 & 73.3 & 73.3 & 70.0 & $524.1 \pm 0.9$ & 524.5 & 522.9 & 525.3 & 523.6 \\
 & 48 & $73.3 \pm 0.0$ & 73.3 & 73.3 & 73.3 & 73.3 & $629.2 \pm 0.7$ & 629.7 & 628.0 & 629.9 & 629.2 \\
 & 56 & $73.3 \pm 0.0$ & 73.3 & 73.3 & 73.3 & 73.3 & $734.8 \pm 0.9$ & 734.6 & 733.6 & 736.0 & 735.2 \\
 & 64 & $73.3 \pm 0.0$ & 73.3 & 73.3 & 73.3 & 73.3 & $840.1 \pm 1.0$ & 839.3 & 839.1 & 841.5 & 840.5 \\
\midrule
\multirow{10}{*}{\shortstack[l]{Multi-Turn\\Refinement}} & 4 & $62.5 \pm 2.8$ & 63.3 & 60.0 & 60.0 & 66.7 & $42.5 \pm 0.8$ & 42.6 & 43.4 & 43.0 & 41.1 \\
 & 8 & $70.8 \pm 4.9$ & 73.3 & 63.3 & 70.0 & 76.7 & $73.4 \pm 1.5$ & 72.8 & 75.9 & 72.7 & 72.0 \\
 & 12 & $72.5 \pm 4.9$ & 73.3 & 66.7 & 70.0 & 80.0 & $100.4 \pm 2.3$ & 101.1 & 103.8 & 98.3 & 98.3 \\
 & 16 & $73.3 \pm 5.3$ & 76.7 & 66.7 & 70.0 & 80.0 & $125.7 \pm 3.4$ & 127.2 & 130.4 & 121.2 & 124.1 \\
 & 24 & $75.0 \pm 5.0$ & 80.0 & 70.0 & 70.0 & 80.0 & $170.5 \pm 4.2$ & 172.5 & 174.8 & 163.5 & 171.1 \\
 & 32 & $75.0 \pm 5.0$ & 80.0 & 70.0 & 70.0 & 80.0 & $207.4 \pm 3.9$ & 206.1 & 209.8 & 201.7 & 211.9 \\
 & 40 & $75.8 \pm 4.3$ & 80.0 & 73.3 & 70.0 & 80.0 & $242.6 \pm 5.7$ & 237.6 & 245.6 & 236.7 & 250.5 \\
 & 48 & $75.8 \pm 4.3$ & 80.0 & 73.3 & 70.0 & 80.0 & $275.8 \pm 6.6$ & 269.4 & 279.0 & 269.8 & 285.1 \\
 & 56 & $75.8 \pm 4.3$ & 80.0 & 73.3 & 70.0 & 80.0 & $308.8 \pm 7.0$ & 301.1 & 312.5 & 303.1 & 318.4 \\
 & 64 & $76.7 \pm 3.3$ & 80.0 & 73.3 & 73.3 & 80.0 & $341.9 \pm 7.1$ & 333.2 & 345.5 & 337.3 & 351.5 \\
\midrule
\multirow{4}{*}{RSA} & 16 ($t_{0}$) & $70.0 \pm 2.4$ & 73.3 & 70.0 & 70.0 & 66.7 & $210.3 \pm 1.0$ & 211.5 & 210.9 & 208.8 & 210.0 \\
 & 16 ($t_{1}$) & $72.5 \pm 2.8$ & 73.3 & 70.0 & 70.0 & 76.7 & $338.2 \pm 2.0$ & 341.0 & 338.6 & 335.6 & 337.6 \\
 & 16 ($t_{2}$) & $77.5 \pm 2.8$ & 80.0 & 76.7 & 73.3 & 80.0 & $445.6 \pm 4.2$ & 452.9 & 442.6 & 442.8 & 444.1 \\
 & 16 ($t_{3}$) & $75.0 \pm 3.7$ & 80.0 & 73.3 & 70.0 & 76.7 & $538.4 \pm 8.3$ & 552.7 & 533.4 & 534.9 & 532.5 \\
\bottomrule
\end{tabular}
\vspace{1mm}
\caption{Inference scaling and error analysis for Qwen3-8B on AIME 2025. We evaluate pass@$k$ (\%) and average generation cost per question over 4 different seeds.}
\label{tab:detailed_performance_S16_Qwen3-8B_k64_aime25}
\end{table}

\begin{table}[hp]
\centering
\small
\setlength{\tabcolsep}{4pt}
\begin{tabular}{l c crrrr crrrr}
\toprule
 &  & \multicolumn{5}{c}{\textbf{pass@$k$ (\%)}} & \multicolumn{5}{c}{\textbf{Generated Tokens ($\times 10^3$)}} \\
\cmidrule(lr){3-7}\cmidrule(lr){8-12}
Method & $k$ & Mean $\pm$ SD & $s$=0 & $s$=1 & $s$=2 & $s$=3 & Mean $\pm$ SD & $s$=0 & $s$=1 & $s$=2 & $s$=3 \\
\midrule
\multirow{10}{*}{TTEL} & 4 & $41.7 \pm 3.7$ & 40.0 & 46.7 & 43.3 & 36.7 & $40.9 \pm 1.5$ & 42.4 & 38.4 & 41.9 & 41.0 \\
 & 8 & $47.5 \pm 5.5$ & 43.3 & 56.7 & 43.3 & 46.7 & $75.1 \pm 2.1$ & 76.5 & 71.4 & 75.9 & 76.6 \\
 & 12 & $51.7 \pm 5.0$ & 50.0 & 60.0 & 50.0 & 46.7 & $110.2 \pm 3.3$ & 109.5 & 105.2 & 112.4 & 113.8 \\
 & 16 & $54.2 \pm 3.6$ & 53.3 & 60.0 & 53.3 & 50.0 & $146.6 \pm 4.5$ & 142.8 & 141.5 & 150.1 & 152.0 \\
 & 24 & $55.8 \pm 4.9$ & 56.7 & 63.3 & 53.3 & 50.0 & $216.7 \pm 8.3$ & 207.9 & 212.2 & 216.5 & 230.0 \\
 & 32 & $57.5 \pm 3.6$ & 56.7 & 63.3 & 53.3 & 56.7 & $280.7 \pm 15.6$ & 265.1 & 269.7 & 283.0 & 305.2 \\
 & 40 & $60.0 \pm 4.1$ & 56.7 & 66.7 & 56.7 & 60.0 & $344.6 \pm 22.2$ & 322.1 & 328.4 & 348.9 & 379.0 \\
 & 48 & $60.8 \pm 3.6$ & 56.7 & 66.7 & 60.0 & 60.0 & $408.2 \pm 32.8$ & 379.5 & 377.4 & 418.7 & 457.4 \\
 & 56 & $60.8 \pm 3.6$ & 56.7 & 66.7 & 60.0 & 60.0 & $469.5 \pm 42.7$ & 434.3 & 428.4 & 480.5 & 534.8 \\
 & 64 & $60.8 \pm 3.6$ & 56.7 & 66.7 & 60.0 & 60.0 & $527.9 \pm 52.3$ & 490.0 & 474.9 & 537.6 & 609.3 \\
\midrule
\multirow{10}{*}{\shortstack[l]{Independent\\Sampling}} & 4 & $40.0 \pm 2.4$ & 40.0 & 36.7 & 43.3 & 40.0 & $57.4 \pm 0.6$ & 57.1 & 57.8 & 56.5 & 58.0 \\
 & 8 & $42.5 \pm 3.6$ & 46.7 & 36.7 & 43.3 & 43.3 & $114.7 \pm 0.8$ & 114.7 & 115.0 & 113.4 & 115.6 \\
 & 12 & $43.3 \pm 4.7$ & 50.0 & 36.7 & 43.3 & 43.3 & $171.8 \pm 0.8$ & 171.4 & 172.2 & 170.8 & 172.8 \\
 & 16 & $45.0 \pm 3.7$ & 50.0 & 40.0 & 43.3 & 46.7 & $228.8 \pm 1.0$ & 227.5 & 229.3 & 228.2 & 230.2 \\
 & 24 & $45.8 \pm 2.8$ & 50.0 & 43.3 & 43.3 & 46.7 & $343.3 \pm 1.8$ & 341.1 & 343.4 & 342.7 & 346.0 \\
 & 32 & $49.2 \pm 3.6$ & 53.3 & 50.0 & 43.3 & 50.0 & $458.0 \pm 1.7$ & 456.2 & 458.3 & 456.8 & 460.7 \\
 & 40 & $51.7 \pm 5.5$ & 56.7 & 56.7 & 43.3 & 50.0 & $572.5 \pm 2.3$ & 570.1 & 572.5 & 571.1 & 576.1 \\
 & 48 & $54.2 \pm 8.6$ & 66.7 & 56.7 & 43.3 & 50.0 & $687.5 \pm 2.8$ & 684.6 & 688.0 & 685.7 & 691.8 \\
 & 56 & $55.0 \pm 9.0$ & 66.7 & 60.0 & 43.3 & 50.0 & $802.2 \pm 2.9$ & 799.4 & 801.9 & 800.6 & 806.9 \\
 & 64 & $55.8 \pm 8.6$ & 66.7 & 60.0 & 43.3 & 53.3 & $917.1 \pm 2.7$ & 914.5 & 916.7 & 915.8 & 921.7 \\
\midrule
\multirow{10}{*}{\shortstack[l]{Multi-Turn\\Refinement}} & 4 & $41.7 \pm 1.7$ & 43.3 & 43.3 & 40.0 & 40.0 & $45.6 \pm 1.1$ & 46.5 & 45.1 & 46.7 & 44.0 \\
 & 8 & $46.7 \pm 5.3$ & 43.3 & 53.3 & 50.0 & 40.0 & $74.2 \pm 3.2$ & 77.3 & 70.6 & 77.3 & 71.5 \\
 & 12 & $49.2 \pm 3.6$ & 50.0 & 53.3 & 50.0 & 43.3 & $96.9 \pm 4.5$ & 101.8 & 91.1 & 100.7 & 94.0 \\
 & 16 & $50.0 \pm 4.1$ & 50.0 & 53.3 & 53.3 & 43.3 & $116.4 \pm 6.1$ & 122.8 & 109.1 & 121.9 & 111.7 \\
 & 24 & $51.7 \pm 2.9$ & 53.3 & 53.3 & 53.3 & 46.7 & $153.1 \pm 7.4$ & 160.4 & 143.9 & 160.4 & 147.6 \\
 & 32 & $53.3 \pm 4.1$ & 56.7 & 53.3 & 56.7 & 46.7 & $187.3 \pm 7.5$ & 193.8 & 178.3 & 195.6 & 181.5 \\
 & 40 & $53.3 \pm 4.1$ & 56.7 & 53.3 & 56.7 & 46.7 & $220.6 \pm 8.0$ & 226.8 & 210.4 & 229.8 & 215.3 \\
 & 48 & $53.3 \pm 4.1$ & 56.7 & 53.3 & 56.7 & 46.7 & $253.7 \pm 6.6$ & 259.8 & 244.8 & 260.1 & 250.0 \\
 & 56 & $54.2 \pm 2.8$ & 56.7 & 53.3 & 56.7 & 50.0 & $287.5 \pm 4.8$ & 292.5 & 279.9 & 290.2 & 287.4 \\
 & 64 & $54.2 \pm 2.8$ & 56.7 & 53.3 & 56.7 & 50.0 & $320.9 \pm 4.6$ & 326.5 & 313.8 & 322.5 & 320.8 \\
\midrule
\multirow{4}{*}{RSA} & 16 ($t_{0}$) & $47.5 \pm 6.4$ & 43.3 & 40.0 & 50.0 & 56.7 & $229.0 \pm 0.4$ & 229.3 & 229.2 & 228.3 & 229.4 \\
 & 16 ($t_{1}$) & $56.7 \pm 2.4$ & 56.7 & 56.7 & 53.3 & 60.0 & $365.0 \pm 3.7$ & 358.6 & 367.2 & 366.0 & 368.1 \\
 & 16 ($t_{2}$) & $54.2 \pm 2.8$ & 56.7 & 53.3 & 50.0 & 56.7 & $461.8 \pm 4.9$ & 453.4 & 464.2 & 465.5 & 464.1 \\
 & 16 ($t_{3}$) & $49.2 \pm 3.6$ & 50.0 & 43.3 & 50.0 & 53.3 & $547.5 \pm 4.7$ & 539.5 & 550.9 & 551.0 & 548.6 \\
\bottomrule
\end{tabular}
\vspace{1mm}
\caption{Inference scaling and error analysis for Qwen3-8B on HMMT. We evaluate pass@$k$ (\%) and average generation cost per question over 4 different seeds.}
\label{tab:detailed_performance_S17_Qwen3-8B_k64_hmmt}
\end{table}

\begin{table}[hp]
\centering
\small
\setlength{\tabcolsep}{4pt}
\begin{tabular}{l c crrr crrr}
\toprule
 &  & \multicolumn{4}{c}{\textbf{pass@$k$ (\%)}} & \multicolumn{4}{c}{\textbf{Generated Tokens ($\times 10^3$)}} \\
\cmidrule(lr){3-6}\cmidrule(lr){7-10}
Method & $k$ & Mean $\pm$ SD & $s$=0 & $s$=1 & $s$=2 & Mean $\pm$ SD & $s$=0 & $s$=1 & $s$=2 \\
\midrule
\multirow{10}{*}{TTEL} & 4 & $51.7 \pm 2.5$ & 53.4 & 53.4 & 48.1 & $25.4 \pm 0.8$ & 26.5 & 25.1 & 24.7 \\
 & 8 & $57.0 \pm 1.6$ & 58.8 & 57.3 & 55.0 & $47.3 \pm 0.6$ & 46.7 & 47.1 & 48.1 \\
 & 12 & $59.3 \pm 1.4$ & 60.3 & 60.3 & 57.3 & $67.9 \pm 1.1$ & 66.7 & 67.6 & 69.3 \\
 & 16 & $61.8 \pm 1.1$ & 63.4 & 61.1 & 61.1 & $87.9 \pm 0.8$ & 86.8 & 88.0 & 88.8 \\
 & 24 & $63.9 \pm 1.6$ & 65.6 & 61.8 & 64.1 & $126.6 \pm 2.5$ & 123.1 & 128.3 & 128.4 \\
 & 32 & $65.9 \pm 1.3$ & 66.4 & 64.1 & 67.2 & $165.2 \pm 4.6$ & 159.2 & 170.3 & 166.1 \\
 & 40 & $67.2 \pm 0.6$ & 67.9 & 66.4 & 67.2 & $203.4 \pm 6.4$ & 195.2 & 210.8 & 204.1 \\
 & 48 & $68.4 \pm 0.4$ & 67.9 & 68.7 & 68.7 & $239.9 \pm 8.2$ & 229.4 & 249.3 & 241.0 \\
 & 56 & $69.0 \pm 0.7$ & 67.9 & 69.5 & 69.5 & $276.6 \pm 9.4$ & 265.0 & 288.0 & 276.8 \\
 & 64 & $69.5 \pm 0.6$ & 68.7 & 70.2 & 69.5 & $312.8 \pm 10.7$ & 299.5 & 325.7 & 313.2 \\
\midrule
\multirow{10}{*}{\shortstack[l]{Independent\\Sampling}} & 4 & $52.4 \pm 1.9$ & 55.0 & 50.4 & 51.9 & $47.0 \pm 0.3$ & 46.7 & 46.9 & 47.5 \\
 & 8 & $56.0 \pm 1.0$ & 57.3 & 55.0 & 55.7 & $93.9 \pm 0.5$ & 93.6 & 93.5 & 94.5 \\
 & 12 & $57.8 \pm 1.3$ & 59.5 & 56.5 & 57.3 & $140.8 \pm 0.3$ & 140.5 & 140.8 & 141.2 \\
 & 16 & $60.1 \pm 0.7$ & 61.1 & 59.5 & 59.5 & $187.7 \pm 0.3$ & 187.4 & 187.8 & 188.0 \\
 & 24 & $61.1 \pm 1.2$ & 62.6 & 61.1 & 59.5 & $281.6 \pm 0.2$ & 281.3 & 281.5 & 281.9 \\
 & 32 & $64.4 \pm 2.2$ & 67.2 & 61.8 & 64.1 & $375.2 \pm 0.5$ & 374.8 & 374.9 & 375.8 \\
 & 40 & $65.4 \pm 1.9$ & 67.9 & 63.4 & 64.9 & $468.7 \pm 0.5$ & 468.5 & 468.3 & 469.4 \\
 & 48 & $65.6 \pm 1.6$ & 67.9 & 64.1 & 64.9 & $562.7 \pm 0.3$ & 562.3 & 562.7 & 563.2 \\
 & 56 & $66.7 \pm 1.0$ & 67.9 & 65.6 & 66.4 & $656.1 \pm 0.4$ & 655.6 & 656.4 & 656.4 \\
 & 64 & $67.4 \pm 1.4$ & 69.5 & 66.4 & 66.4 & $749.9 \pm 0.1$ & 749.9 & 750.1 & 749.8 \\
\midrule
\multirow{10}{*}{\shortstack[l]{Multi-Turn\\Refinement}} & 4 & $50.6 \pm 2.5$ & 53.4 & 51.1 & 47.3 & $29.8 \pm 0.3$ & 29.3 & 29.8 & 30.1 \\
 & 8 & $51.1 \pm 2.9$ & 54.2 & 51.9 & 47.3 & $49.4 \pm 1.1$ & 47.9 & 50.2 & 50.2 \\
 & 12 & $52.7 \pm 2.7$ & 55.0 & 54.2 & 48.9 & $72.0 \pm 1.7$ & 69.7 & 73.5 & 72.7 \\
 & 16 & $53.9 \pm 2.2$ & 56.5 & 54.2 & 51.1 & $94.0 \pm 1.2$ & 92.3 & 94.8 & 94.9 \\
 & 24 & $55.0 \pm 1.1$ & 56.5 & 54.2 & 54.2 & $132.1 \pm 0.7$ & 131.2 & 132.1 & 133.0 \\
 & 32 & $55.0 \pm 1.1$ & 56.5 & 54.2 & 54.2 & $167.1 \pm 0.2$ & 166.9 & 167.0 & 167.4 \\
 & 40 & $55.0 \pm 1.1$ & 56.5 & 54.2 & 54.2 & $200.6 \pm 0.6$ & 200.4 & 200.0 & 201.4 \\
 & 48 & $55.2 \pm 1.0$ & 56.5 & 55.0 & 54.2 & $233.1 \pm 0.2$ & 232.7 & 233.1 & 233.3 \\
 & 56 & $55.7 \pm 0.6$ & 56.5 & 55.7 & 55.0 & $265.6 \pm 0.8$ & 264.4 & 266.3 & 266.1 \\
 & 64 & $56.0 \pm 1.0$ & 57.3 & 55.7 & 55.0 & $298.1 \pm 1.5$ & 296.0 & 299.4 & 299.0 \\
\midrule
\multirow{4}{*}{RSA} & 16 ($t_{0}$) & $58.8 \pm 1.2$ & 58.8 & 60.3 & 57.3 & $187.2 \pm 0.2$ & 187.0 & 187.2 & 187.5 \\
 & 16 ($t_{1}$) & $64.4 \pm 1.0$ & 63.4 & 64.1 & 65.6 & $269.0 \pm 1.0$ & 270.2 & 267.7 & 269.2 \\
 & 16 ($t_{2}$) & $63.6 \pm 0.4$ & 64.1 & 63.4 & 63.4 & $415.7 \pm 0.8$ & 415.3 & 415.1 & 416.8 \\
 & 16 ($t_{3}$) & $60.8 \pm 0.7$ & 60.3 & 61.8 & 60.3 & $522.9 \pm 2.3$ & 522.8 & 520.2 & 525.7 \\
\bottomrule
\end{tabular}
\vspace{1mm}
\caption{Inference scaling and error analysis for Qwen3-4B-Thinking-2507 on LiveCodeBench (LCB). We evaluate pass@$k$ (\%) and average generation cost per question over 3 different seeds.}
\label{tab:detailed_performance_S10_Qwen3-4B-Think_k64_lcb}
\end{table}

\begin{table}[hp]
\centering
\small
\setlength{\tabcolsep}{4pt}
\begin{tabular}{l c crrrr crrrr}
\toprule
 &  & \multicolumn{5}{c}{\textbf{pass@$k$ (\%)}} & \multicolumn{5}{c}{\textbf{Generated Tokens ($\times 10^3$)}} \\
\cmidrule(lr){3-7}\cmidrule(lr){8-12}
Method & $k$ & Mean $\pm$ SD & $s$=0 & $s$=1 & $s$=2 & $s$=3 & Mean $\pm$ SD & $s$=0 & $s$=1 & $s$=2 & $s$=3 \\
\midrule
\multirow{10}{*}{TTEL} & 4 & $73.3 \pm 2.4$ & 76.7 & 73.3 & 70.0 & 73.3 & $40.3 \pm 0.5$ & 39.8 & 40.4 & 40.0 & 41.0 \\
 & 8 & $80.8 \pm 2.8$ & 83.3 & 76.7 & 83.3 & 80.0 & $76.7 \pm 1.1$ & 77.4 & 75.6 & 75.6 & 78.1 \\
 & 12 & $85.0 \pm 1.7$ & 86.7 & 83.3 & 86.7 & 83.3 & $108.2 \pm 2.1$ & 109.8 & 104.7 & 109.7 & 108.4 \\
 & 16 & $85.0 \pm 1.7$ & 86.7 & 83.3 & 86.7 & 83.3 & $141.9 \pm 4.3$ & 144.0 & 136.5 & 147.6 & 139.4 \\
 & 24 & $85.0 \pm 1.7$ & 86.7 & 83.3 & 86.7 & 83.3 & $205.4 \pm 8.4$ & 206.4 & 198.9 & 218.7 & 197.5 \\
 & 32 & $85.0 \pm 1.7$ & 86.7 & 83.3 & 86.7 & 83.3 & $264.7 \pm 11.8$ & 259.6 & 261.7 & 284.4 & 253.1 \\
 & 40 & $85.0 \pm 1.7$ & 86.7 & 83.3 & 86.7 & 83.3 & $323.6 \pm 15.4$ & 311.0 & 321.7 & 349.4 & 312.4 \\
 & 48 & $85.0 \pm 1.7$ & 86.7 & 83.3 & 86.7 & 83.3 & $379.4 \pm 22.2$ & 353.1 & 380.9 & 413.9 & 369.9 \\
 & 56 & $85.0 \pm 1.7$ & 86.7 & 83.3 & 86.7 & 83.3 & $437.0 \pm 28.5$ & 404.6 & 433.2 & 482.7 & 427.3 \\
 & 64 & $85.0 \pm 1.7$ & 86.7 & 83.3 & 86.7 & 83.3 & $491.0 \pm 33.0$ & 451.9 & 481.6 & 543.2 & 487.4 \\
\midrule
\multirow{10}{*}{\shortstack[l]{Independent\\Sampling}} & 4 & $43.3 \pm 4.1$ & 36.7 & 43.3 & 46.7 & 46.7 & $58.1 \pm 0.5$ & 59.0 & 57.8 & 57.8 & 58.0 \\
 & 8 & $50.0 \pm 3.3$ & 46.7 & 53.3 & 46.7 & 53.3 & $116.3 \pm 0.7$ & 117.1 & 115.7 & 116.8 & 115.7 \\
 & 12 & $50.8 \pm 2.8$ & 46.7 & 53.3 & 50.0 & 53.3 & $174.1 \pm 0.5$ & 174.9 & 173.8 & 174.1 & 173.7 \\
 & 16 & $51.7 \pm 3.7$ & 46.7 & 53.3 & 50.0 & 56.7 & $232.7 \pm 0.8$ & 234.0 & 232.5 & 232.4 & 231.7 \\
 & 24 & $56.7 \pm 4.1$ & 50.0 & 60.0 & 56.7 & 60.0 & $349.0 \pm 0.6$ & 349.8 & 348.1 & 349.2 & 348.7 \\
 & 32 & $58.3 \pm 2.9$ & 53.3 & 60.0 & 60.0 & 60.0 & $464.7 \pm 0.7$ & 463.6 & 464.8 & 465.4 & 464.8 \\
 & 40 & $60.0 \pm 2.4$ & 56.7 & 60.0 & 63.3 & 60.0 & $581.0 \pm 1.0$ & 579.8 & 582.1 & 580.2 & 581.7 \\
 & 48 & $60.8 \pm 2.8$ & 56.7 & 63.3 & 63.3 & 60.0 & $696.2 \pm 1.9$ & 694.4 & 698.1 & 694.3 & 698.0 \\
 & 56 & $60.8 \pm 2.8$ & 56.7 & 63.3 & 63.3 & 60.0 & $812.2 \pm 1.5$ & 810.8 & 813.9 & 810.5 & 813.5 \\
 & 64 & $61.7 \pm 1.7$ & 60.0 & 63.3 & 63.3 & 60.0 & $928.9 \pm 2.8$ & 926.9 & 931.3 & 925.4 & 932.0 \\
\midrule
\multirow{10}{*}{\shortstack[l]{Multi-Turn\\Refinement}} & 4 & $71.7 \pm 1.7$ & 73.3 & 73.3 & 70.0 & 70.0 & $31.2 \pm 0.8$ & 32.0 & 30.6 & 30.3 & 32.0 \\
 & 8 & $71.7 \pm 1.7$ & 73.3 & 73.3 & 70.0 & 70.0 & $45.8 \pm 0.8$ & 46.5 & 44.7 & 45.3 & 46.7 \\
 & 12 & $75.0 \pm 1.7$ & 76.7 & 76.7 & 73.3 & 73.3 & $64.1 \pm 1.3$ & 65.7 & 63.5 & 62.4 & 65.0 \\
 & 16 & $75.0 \pm 1.7$ & 76.7 & 76.7 & 73.3 & 73.3 & $84.0 \pm 1.9$ & 85.1 & 83.1 & 81.5 & 86.4 \\
 & 24 & $75.0 \pm 1.7$ & 76.7 & 76.7 & 73.3 & 73.3 & $119.0 \pm 2.8$ & 120.3 & 116.4 & 116.3 & 122.9 \\
 & 32 & $76.7 \pm 3.3$ & 80.0 & 80.0 & 73.3 & 73.3 & $153.2 \pm 5.3$ & 154.2 & 147.4 & 149.9 & 161.4 \\
 & 40 & $76.7 \pm 3.3$ & 80.0 & 80.0 & 73.3 & 73.3 & $188.3 \pm 6.3$ & 189.6 & 181.2 & 184.5 & 197.8 \\
 & 48 & $78.3 \pm 1.7$ & 80.0 & 80.0 & 76.7 & 76.7 & $224.4 \pm 6.3$ & 227.3 & 216.4 & 221.0 & 232.9 \\
 & 56 & $79.2 \pm 1.4$ & 80.0 & 80.0 & 76.7 & 80.0 & $261.1 \pm 7.5$ & 265.4 & 251.8 & 256.4 & 271.0 \\
 & 64 & $80.0 \pm 2.4$ & 83.3 & 80.0 & 76.7 & 80.0 & $297.8 \pm 7.3$ & 302.9 & 288.4 & 293.4 & 306.7 \\
\midrule
\multirow{4}{*}{RSA} & 16 ($t_{0}$) & $53.3 \pm 4.7$ & 53.3 & 60.0 & 46.7 & 53.3 & $232.6 \pm 0.6$ & 232.6 & 231.8 & 233.5 & 232.6 \\
 & 16 ($t_{1}$) & $69.2 \pm 2.8$ & 66.7 & 70.0 & 73.3 & 66.7 & $268.4 \pm 1.8$ & 270.1 & 266.4 & 266.9 & 270.3 \\
 & 16 ($t_{2}$) & $58.3 \pm 3.7$ & 63.3 & 56.7 & 60.0 & 53.3 & $375.7 \pm 4.3$ & 378.8 & 369.2 & 374.6 & 380.3 \\
 & 16 ($t_{3}$) & $53.3 \pm 4.1$ & 56.7 & 53.3 & 56.7 & 46.7 & $449.2 \pm 5.5$ & 446.4 & 444.7 & 447.1 & 458.6 \\
\bottomrule
\end{tabular}
\vspace{1mm}
\caption{Inference scaling and error analysis for Qwen3-4B-Thinking-2507 on AIME 2025. We evaluate pass@$k$ (\%) and average generation cost per question over 4 different seeds.}
\label{tab:detailed_performance_S18_Qwen3-4B-Think_k64_aime25}
\end{table}

\begin{table}[hp]
\centering
\small
\setlength{\tabcolsep}{4pt}
\begin{tabular}{l c crrrr crrrr}
\toprule
 &  & \multicolumn{5}{c}{\textbf{pass@$k$ (\%)}} & \multicolumn{5}{c}{\textbf{Generated Tokens ($\times 10^3$)}} \\
\cmidrule(lr){3-7}\cmidrule(lr){8-12}
Method & $k$ & Mean $\pm$ SD & $s$=0 & $s$=1 & $s$=2 & $s$=3 & Mean $\pm$ SD & $s$=0 & $s$=1 & $s$=2 & $s$=3 \\
\midrule
\multirow{10}{*}{TTEL} & 4 & $45.8 \pm 8.3$ & 53.3 & 43.3 & 53.3 & 33.3 & $43.3 \pm 0.7$ & 42.6 & 42.6 & 43.8 & 44.2 \\
 & 8 & $55.8 \pm 4.9$ & 56.7 & 53.3 & 63.3 & 50.0 & $81.5 \pm 0.9$ & 80.5 & 81.3 & 82.9 & 81.2 \\
 & 12 & $60.0 \pm 5.3$ & 63.3 & 53.3 & 66.7 & 56.7 & $115.4 \pm 1.3$ & 116.0 & 113.2 & 116.0 & 116.4 \\
 & 16 & $62.5 \pm 4.3$ & 66.7 & 56.7 & 66.7 & 60.0 & $147.9 \pm 3.5$ & 151.1 & 145.4 & 151.5 & 143.5 \\
 & 24 & $64.2 \pm 2.8$ & 66.7 & 60.0 & 66.7 & 63.3 & $210.7 \pm 7.8$ & 217.9 & 200.9 & 218.7 & 205.3 \\
 & 32 & $65.0 \pm 1.7$ & 66.7 & 63.3 & 66.7 & 63.3 & $269.6 \pm 10.6$ & 279.9 & 256.2 & 280.1 & 262.2 \\
 & 40 & $68.3 \pm 2.9$ & 70.0 & 63.3 & 70.0 & 70.0 & $325.1 \pm 14.2$ & 338.2 & 305.8 & 339.2 & 317.4 \\
 & 48 & $68.3 \pm 2.9$ & 70.0 & 63.3 & 70.0 & 70.0 & $375.6 \pm 15.7$ & 390.5 & 355.9 & 391.4 & 364.4 \\
 & 56 & $69.2 \pm 3.6$ & 70.0 & 63.3 & 73.3 & 70.0 & $424.4 \pm 18.8$ & 438.5 & 403.4 & 447.4 & 408.5 \\
 & 64 & $69.2 \pm 3.6$ & 70.0 & 63.3 & 73.3 & 70.0 & $472.5 \pm 23.1$ & 487.7 & 447.5 & 502.2 & 452.6 \\
\midrule
\multirow{10}{*}{\shortstack[l]{Independent\\Sampling}} & 4 & $28.3 \pm 1.7$ & 26.7 & 30.0 & 30.0 & 26.7 & $61.7 \pm 0.5$ & 62.1 & 60.9 & 62.2 & 61.6 \\
 & 8 & $32.5 \pm 1.4$ & 33.3 & 33.3 & 33.3 & 30.0 & $123.1 \pm 0.9$ & 123.6 & 121.5 & 123.9 & 123.5 \\
 & 12 & $34.2 \pm 1.4$ & 36.7 & 33.3 & 33.3 & 33.3 & $184.7 \pm 1.0$ & 184.8 & 183.1 & 185.6 & 185.3 \\
 & 16 & $35.0 \pm 1.7$ & 36.7 & 36.7 & 33.3 & 33.3 & $246.7 \pm 1.0$ & 246.6 & 245.1 & 247.8 & 247.1 \\
 & 24 & $40.0 \pm 2.4$ & 43.3 & 40.0 & 40.0 & 36.7 & $369.5 \pm 1.3$ & 369.4 & 367.7 & 371.3 & 369.5 \\
 & 32 & $42.5 \pm 3.6$ & 46.7 & 43.3 & 43.3 & 36.7 & $492.4 \pm 0.9$ & 491.2 & 492.0 & 493.7 & 492.7 \\
 & 40 & $45.0 \pm 5.0$ & 46.7 & 50.0 & 46.7 & 36.7 & $615.0 \pm 0.4$ & 614.4 & 615.0 & 615.0 & 615.6 \\
 & 48 & $45.0 \pm 5.0$ & 46.7 & 50.0 & 46.7 & 36.7 & $737.7 \pm 0.9$ & 736.6 & 737.5 & 737.4 & 739.2 \\
 & 56 & $46.7 \pm 4.1$ & 50.0 & 50.0 & 46.7 & 40.0 & $860.9 \pm 0.9$ & 859.7 & 860.8 & 860.6 & 862.3 \\
 & 64 & $47.5 \pm 2.8$ & 50.0 & 50.0 & 46.7 & 43.3 & $984.2 \pm 1.3$ & 982.8 & 985.5 & 982.8 & 985.5 \\
\midrule
\multirow{10}{*}{\shortstack[l]{Multi-Turn\\Refinement}} & 4 & $43.3 \pm 4.1$ & 40.0 & 43.3 & 50.0 & 40.0 & $33.9 \pm 0.5$ & 33.3 & 33.9 & 34.5 & 34.1 \\
 & 8 & $45.0 \pm 2.9$ & 43.3 & 43.3 & 50.0 & 43.3 & $49.1 \pm 0.4$ & 48.7 & 49.0 & 49.1 & 49.7 \\
 & 12 & $45.0 \pm 2.9$ & 43.3 & 43.3 & 50.0 & 43.3 & $67.5 \pm 0.3$ & 68.0 & 67.6 & 67.0 & 67.5 \\
 & 16 & $45.0 \pm 2.9$ & 43.3 & 43.3 & 50.0 & 43.3 & $89.7 \pm 1.1$ & 88.9 & 91.7 & 89.3 & 88.9 \\
 & 24 & $45.0 \pm 2.9$ & 43.3 & 43.3 & 50.0 & 43.3 & $127.7 \pm 2.6$ & 127.8 & 131.9 & 125.9 & 125.3 \\
 & 32 & $45.0 \pm 2.9$ & 43.3 & 43.3 & 50.0 & 43.3 & $163.8 \pm 3.4$ & 162.9 & 169.0 & 163.6 & 159.6 \\
 & 40 & $45.8 \pm 4.3$ & 43.3 & 43.3 & 53.3 & 43.3 & $202.1 \pm 4.3$ & 199.0 & 208.1 & 204.2 & 197.2 \\
 & 48 & $45.8 \pm 4.3$ & 43.3 & 43.3 & 53.3 & 43.3 & $239.6 \pm 5.3$ & 233.7 & 246.8 & 242.6 & 235.5 \\
 & 56 & $45.8 \pm 4.3$ & 43.3 & 43.3 & 53.3 & 43.3 & $277.7 \pm 6.1$ & 270.2 & 286.3 & 280.0 & 274.5 \\
 & 64 & $45.8 \pm 4.3$ & 43.3 & 43.3 & 53.3 & 43.3 & $316.3 \pm 7.7$ & 306.0 & 326.9 & 319.3 & 313.1 \\
\midrule
\multirow{4}{*}{RSA} & 16 ($t_{0}$) & $33.3 \pm 2.4$ & 33.3 & 30.0 & 36.7 & 33.3 & $246.4 \pm 0.5$ & 246.6 & 245.6 & 246.8 & 246.5 \\
 & 16 ($t_{1}$) & $35.0 \pm 3.7$ & 40.0 & 33.3 & 30.0 & 36.7 & $278.5 \pm 2.0$ & 280.2 & 276.1 & 276.9 & 280.8 \\
 & 16 ($t_{2}$) & $33.3 \pm 5.3$ & 30.0 & 40.0 & 26.7 & 36.7 & $394.9 \pm 2.6$ & 396.7 & 390.5 & 397.0 & 395.5 \\
 & 16 ($t_{3}$) & $21.7 \pm 1.7$ & 23.3 & 20.0 & 20.0 & 23.3 & $473.7 \pm 1.5$ & 475.6 & 474.0 & 473.8 & 471.4 \\
\bottomrule
\end{tabular}
\vspace{1mm}
\caption{Inference scaling and error analysis for Qwen3-4B-Thinking-2507 on HMMT. We evaluate pass@$k$ (\%) and average generation cost per question over 4 different seeds.}
\label{tab:detailed_performance_S19_Qwen3-4B-Think_k64_hmmt}
\end{table}

\FloatBarrier

\section{Qualitative Example}
\label{app:qualitative_example}

\definecolor{spikeT}{HTML}{CCE5FF}   
\definecolor{spikeTB}{HTML}{FFE0B2}  
\definecolor{spikeMagL}{HTML}{DEEBF7} 
\definecolor{spikeMagM}{HTML}{9ECAE1} 
\definecolor{spikeMagH}{HTML}{4292C6} 

\definecolor{spikeMagOrL}{HTML}{FFF4E5} 
\definecolor{spikeMagOrM}{HTML}{FFE4B5} 
\definecolor{spikeMagOrH}{HTML}{FF9900} 

\begin{figure*}[ht!]
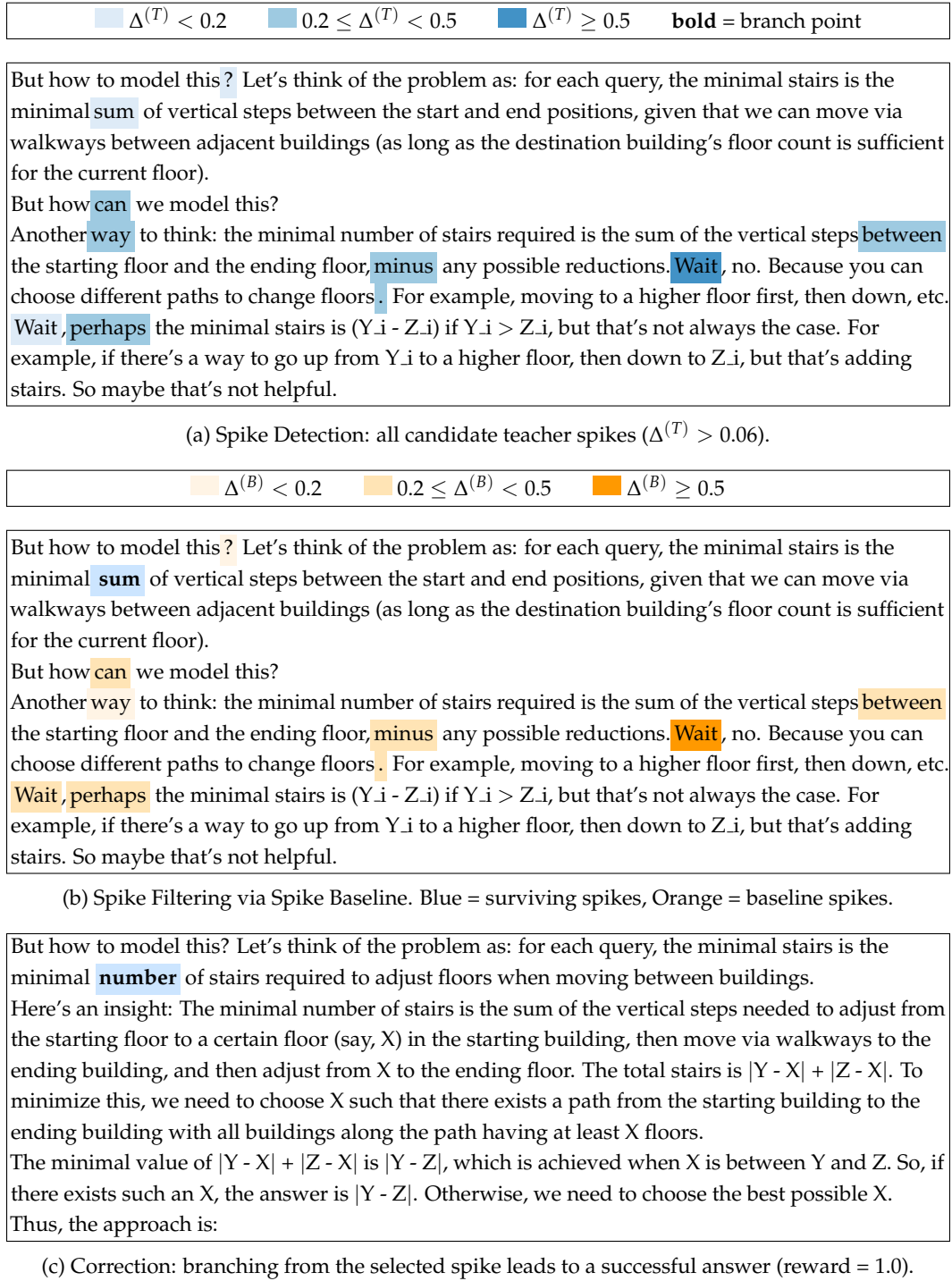

\centering
\setlength{\fboxsep}{1.5pt}

{\setlength{\fboxsep}{1.5pt}%
\fbox{\parbox{\dimexpr\linewidth-2\fboxsep-2\fboxrule\relax}{\centering\footnotesize
\begin{tabular}{@{}c@{\;\;\;\;\;\;\;}c@{\;\;\;\;\;\;\;}c@{\;\;\;\;\;\;\;}c@{}}
{\color{spikeMagL}\rule{12pt}{8pt}}~$\Delta^{(T)} < 0.2$ &
{\color{spikeMagM}\rule{12pt}{8pt}}~$0.2 \le \Delta^{(T)} < 0.5$ &
{\color{spikeMagH}\rule{12pt}{8pt}}~$\Delta^{(T)} \ge 0.5$ &
\textbf{bold}~= branch point
\end{tabular}
}}}
\vspace{0.3em}

\begin{subfigure}[t]{\textwidth}
\centering
\fbox{%
\begin{minipage}{\dimexpr\linewidth-2\fboxsep-2\fboxrule\relax}%
\small
\setlength{\baselineskip}{13pt}
\raggedright
\noindent But how to model this\colorbox{spikeMagL}{\strut ?} Let's think of the problem as: for each query, the minimal stairs is the minimal\colorbox{spikeMagL}{\strut  sum} of vertical steps between the start and end positions, given that we can move via walkways between adjacent buildings (as long as the destination building's floor count is sufficient for the current floor).\par\noindent \par\noindent But how\colorbox{spikeMagM}{\strut  can} we model this?\par\noindent \par\noindent Another\colorbox{spikeMagM}{\strut  way} to think: the minimal number of stairs required is the sum of the vertical steps\colorbox{spikeMagM}{\strut  between} the starting floor and the ending floor,\colorbox{spikeMagM}{\strut  minus} any possible reductions.\colorbox{spikeMagH}{\strut  Wait}, no. Because you can choose different paths to change floors\colorbox{spikeMagM}{\strut .} For example, moving to a higher floor first, then down, etc.\par\noindent \par\noindent \colorbox{spikeMagL}{\strut Wait},\colorbox{spikeMagM}{\strut  perhaps} the minimal stairs is (Y\_i - Z\_i) if Y\_i \textgreater{} Z\_i, but that's not always the case. For example, if there's a way to go up from Y\_i to a higher floor, then down to Z\_i, but that's adding stairs. So maybe that's not helpful.\par\noindent \par\noindent 
\end{minipage}%
}

\caption{Spike Detection: all candidate teacher spikes ($\Delta^{(T)} > 0.06$).}
\label{fig:qual_panel_a}
\end{subfigure}

\vspace{0.8em}

{\setlength{\fboxsep}{1.5pt}%
\fbox{\parbox{\dimexpr\linewidth-2\fboxsep-2\fboxrule\relax}{\centering\footnotesize
\begin{tabular}{@{}c@{\;\;\;\;\;\;\;}c@{\;\;\;\;\;\;\;}c@{\;\;\;\;\;\;\;}c@{}}
{\color{spikeMagOrL}\rule{12pt}{8pt}}~$\Delta^{(B)} < 0.2$ &
{\color{spikeMagOrM}\rule{12pt}{8pt}}~$0.2 \le \Delta^{(B)} < 0.5$ &
{\color{spikeMagOrH}\rule{12pt}{8pt}}~$\Delta^{(B)} \ge 0.5$ &
\end{tabular}
}}}
\vspace{0.3em}

\begin{subfigure}[t]{\textwidth}
\centering
\fbox{%
\begin{minipage}{\dimexpr\linewidth-2\fboxsep-2\fboxrule\relax}%
\small
\setlength{\baselineskip}{13pt}
\raggedright
\noindent But how to model this\colorbox{spikeMagOrL}{\strut ?} Let's think of the problem as: for each query, the minimal stairs is the minimal\colorbox{spikeT}{\strut \textbf{ sum}} of vertical steps between the start and end positions, given that we can move via walkways between adjacent buildings (as long as the destination building's floor count is sufficient for the current floor).\par\noindent \par\noindent But how\colorbox{spikeMagOrM}{\strut  can} we model this?\par\noindent \par\noindent Another\colorbox{spikeMagOrL}{\strut  way} to think: the minimal number of stairs required is the sum of the vertical steps\colorbox{spikeMagOrM}{\strut  between} the starting floor and the ending floor,\colorbox{spikeMagOrM}{\strut  minus} any possible reductions.\colorbox{spikeMagOrH}{\strut  Wait}, no. Because you can choose different paths to change floors\colorbox{spikeMagOrM}{\strut .} For example, moving to a higher floor first, then down, etc.\par\noindent \par\noindent \colorbox{spikeMagOrM}{\strut Wait},\colorbox{spikeMagOrM}{\strut  perhaps} the minimal stairs is (Y\_i - Z\_i) if Y\_i \textgreater{} Z\_i, but that's not always the case. For example, if there's a way to go up from Y\_i to a higher floor, then down to Z\_i, but that's adding stairs. So maybe that's not helpful.\par\noindent \par\noindent 
\end{minipage}%
}

\caption{Spike Filtering via Spike Baseline. Blue = surviving spikes, Orange = baseline spikes.}
\label{fig:qual_panel_b}
\end{subfigure}

\vspace{0.8em}

\begin{subfigure}[t]{\textwidth}
\centering
\fbox{%
\begin{minipage}{\dimexpr\linewidth-2\fboxsep-2\fboxrule\relax}%
\small
\setlength{\baselineskip}{13pt}
\raggedright
\noindent But how to model this? Let's think of the problem as: for each query, the minimal stairs is the minimal \colorbox{spikeT}{\strut \textbf{number}} of stairs required to adjust floors when moving between buildings.\par\noindent \par\noindent Here's an insight: The minimal number of stairs is the sum of the vertical steps needed to adjust from the starting floor to a certain floor (say, X) in the starting building, then move via walkways to the ending building, and then adjust from X to the ending floor. The total stairs is \textbar{}Y - X\textbar{} + \textbar{}Z - X\textbar{}. To minimize this, we need to choose X such that there exists a path from the starting building to the ending building with all buildings along the path having at least X floors.\par\noindent \par\noindent The minimal value of \textbar{}Y - X\textbar{} + \textbar{}Z - X\textbar{} is \textbar{}Y - Z\textbar{}, which is achieved when X is between Y and Z. So, if there exists such an X, the answer is \textbar{}Y - Z\textbar{}. Otherwise, we need to choose the best possible X.\par\noindent \par\noindent Thus, the approach is:\par\noindent 
\end{minipage}%
}

\caption{Correction: branching from the selected spike leads to a successful answer (reward = 1.0).}
\label{fig:qual_panel_c}
\end{subfigure}

\caption{Qualitative analysis of spike detection and correction for Question 21 in LCBv6. }
\label{fig:qualitative_q21}
\end{figure*}

\newpage

The steps of spike detection, spike filtering and branching are illustrated in \cref{fig:qualitative_q21}
\begin{enumerate}[label=\alph*)]
    \item The teacher student probability difference spikes at many points, colored in blue according to magnitude ($\Delta^{(T)}$).
    \item Spike filtering removes tokens where the student-baseline gap ($\Delta^{(B)}$) is also large (colored in orange according to magnitude), isolating only spikes with genuine divergence in reasoning between student and teacher.
    \item When the model re-generates from the chosen surviving spike  (\textbf{sum}), it branches into a correct reasoning path which eventually leads to the correct answer.
\end{enumerate}

\section{Sensitivity Analysis of Spike Threshold}
\label{app:threshold_sensitivity}

\rsc{Fill Text}

\begin{table}[H]
\centering
\small
\begin{tabular}{lccc}
\toprule
$\tau$ & \textbf{pass@$k$ (\%)} & \textbf{Generated Tokens ($\times 10^3$)} & \textbf{Turn 1 Spikes} \\
\midrule
0.04 & 64.4 $\pm$ 0.5 & 93.9 $\pm$ 1.1 & 20.0 $\pm$ 0.6 \\
0.06 & 65.9 $\pm$ 1.6 & 94.3 $\pm$ 1.1 & 17.9 $\pm$ 0.4 \\
0.08 & 64.4 $\pm$ 1.6 & 96.6 $\pm$ 0.7 & 16.7 $\pm$ 0.8 \\
0.1 & 65.4 $\pm$ 2.3 & 95.6 $\pm$ 0.9 & 15.5 $\pm$ 0.3 \\
0.12 & 65.7 $\pm$ 1.3 & 96.7 $\pm$ 0.3 & 14.2 $\pm$ 0.4 \\
0.14 & 64.4 $\pm$ 1.9 & 97.3 $\pm$ 1.0 & 12.9 $\pm$ 0.5 \\
\bottomrule
\end{tabular}
\caption{Hyperparameter sweep for Spike Threshold (Qwen3-8B, LCBv6, $k=16$ rollouts). Mean and standard deviation over 3 seeds.}
\label{tab:threshold-sweep-main}
\end{table}

\begin{table}[H]
\centering
\small
\begin{tabular}{lrrrrrrrrr}
\toprule
 & \multicolumn{3}{c}{\textbf{pass@$k$ (\%)}} & \multicolumn{3}{c}{\textbf{Generated Tokens ($\times 10^3$)}} & \multicolumn{3}{c}{\textbf{Turn 1 Spikes}} \\
\cmidrule(lr){2-4}\cmidrule(lr){5-7}\cmidrule(lr){8-10}
$\tau$ & s=0 & s=1 & s=2 & s=0 & s=1 & s=2 & s=0 & s=1 & s=2 \\
\midrule
0.04 & 64.9 & 64.1 & 64.1 & 93.0 & 93.7 & 95.1 & 20.3 & 20.6 & 19.3 \\
0.06 & 67.2 & 64.1 & 66.4 & 94.6 & 93.1 & 95.2 & 18.3 & 17.8 & 17.6 \\
0.08 & 64.9 & 65.6 & 62.6 & 95.8 & 97.0 & 96.9 & 17.4 & 15.8 & 16.9 \\
0.1 & 67.9 & 63.4 & 64.9 & 94.7 & 95.5 & 96.5 & 15.8 & 15.1 & 15.5 \\
0.12 & 66.4 & 66.4 & 64.1 & 96.4 & 96.9 & 96.9 & 14.7 & 14.0 & 13.9 \\
0.14 & 64.1 & 66.4 & 62.6 & 96.9 & 96.6 & 98.5 & 13.4 & 12.9 & 12.4 \\
\bottomrule
\end{tabular}
\caption{Hyperparameter sweep for Spike Threshold (Qwen3-8B, LCBv6, $k=16$ rollouts). Results per seed (denoted by s)).}
\label{tab:threshold-sweep-appendix}
\end{table}

\section{Branching Analysis}
\label{app:branching_analysis}

\begin{figure}[H]
    \centering
    \includegraphics[width=\linewidth]{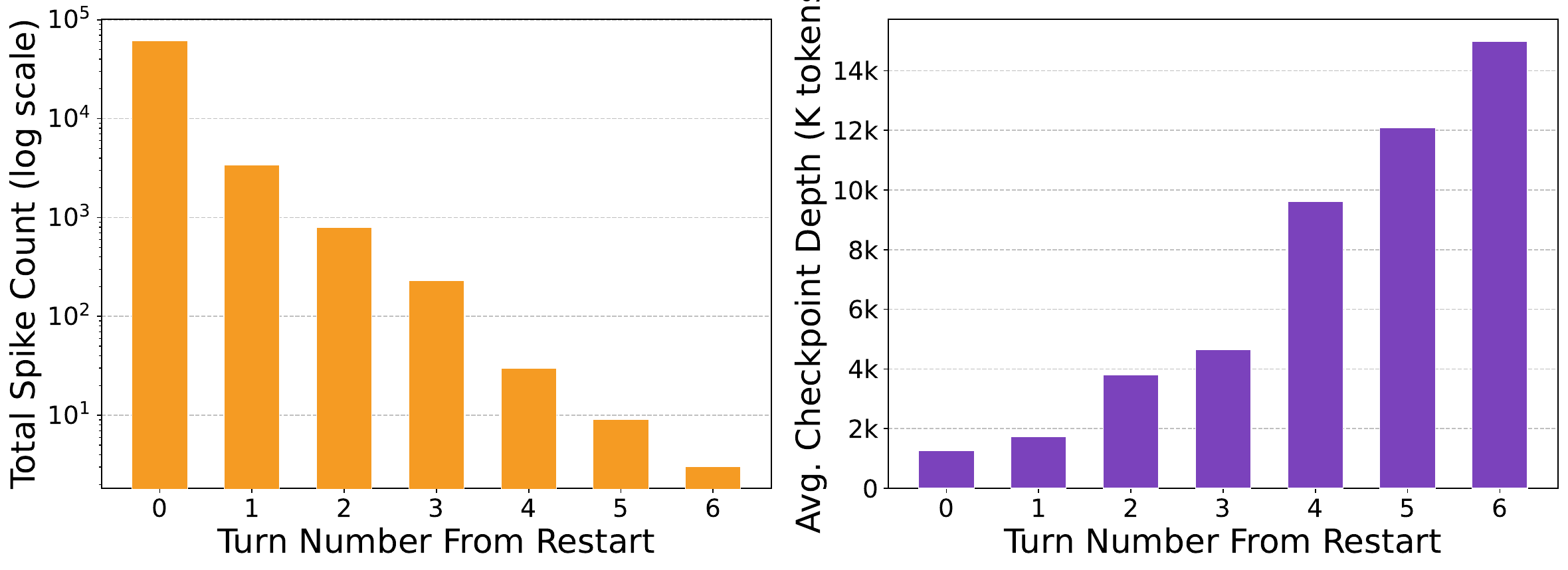}

    \caption{Total spike count and average checkpoint depth for successive turn numbers (relative to restarts), evaluated on LiveCodeBench V6 with Qwen3-8B ($k=16$) using {\ttsd}-GenFB. The number of spikes decreases exponentially between successive turns. \vspace{-0.1in}}
    \label{fig:spike_metrics_by_turn}
\end{figure}

\begin{figure}[H]
    \centering
    \includegraphics[width=\linewidth]{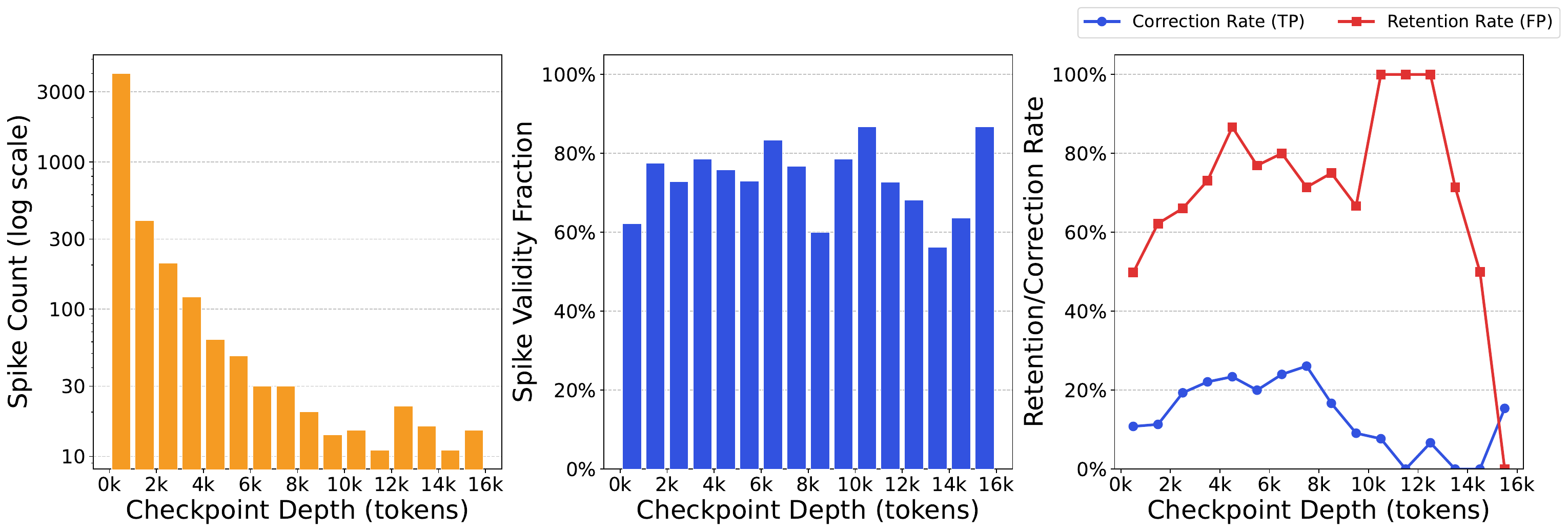}

    \caption{Spike count, validity of spikes and success rates based on Checkpoint Depth collected for {\ttsd}-GenFB on LiveCodeBench V6 with Qwen3-8B ($k=16$). A spike is invalid (or a False Positive) if the answer from that turn is correct. We define the success rate of the branched rollout as the correction rate for valid spikes and the retention rate for invalid spikes. We observe that retention rate is greater than correction rate. Correction rate decreases and retention rate increases after a checkpoint depth of 8k tokens. \vspace{-0.1in}}
    \label{fig:spike_metrics_by_depth}
\end{figure}

\end{document}